\pdfoutput=1

\documentclass[11pt]{article}

\usepackage[preprint]{acl}

\usepackage{times}
\usepackage{latexsym}

\usepackage[T1]{fontenc}

\usepackage[utf8]{inputenc}

\usepackage{microtype}

\usepackage{inconsolata}
\usepackage{graphicx}
\usepackage{amsmath}
\usepackage{enumerate}
\usepackage{array}
\usepackage{booktabs}
\usepackage{multirow}
\usepackage{subcaption}
\usepackage{tcolorbox}
\usepackage{listings}
\usepackage{amssymb}
\usepackage{amsthm}
\usepackage{algorithm}
\usepackage{algorithmic}
\usepackage{bm}
\usepackage{tabularx}
\usepackage{newfloat}
\usepackage{longtable}
\usepackage{algorithm}
\usepackage{algorithmic}
\usepackage{tikz}
\usepackage{bm}
\usepackage{subcaption}
\usepackage{multicol}
\usepackage{tabularx}
\newtheorem{definition}{Definition}
\newtheorem{theorem}[]{Theorem}
\usepackage{makecell}

\title{FIRE\raisebox{-0.27ex}{\tikz \node[scale=1.0, inner sep=0pt, outer sep=0pt] {\includegraphics[height=1em]{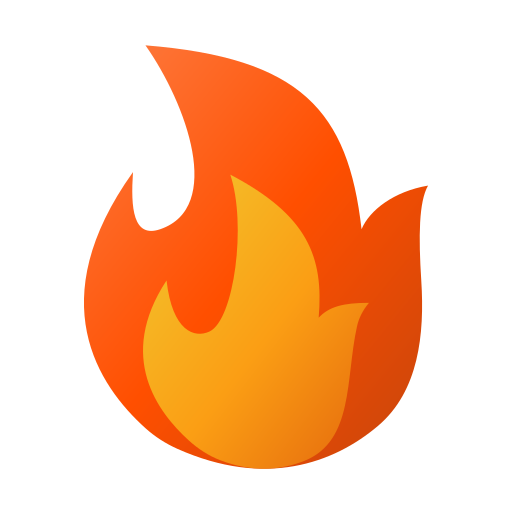}};}: Flexible Integration of Data Quality Ratings for Effective Pretraining}

\author{
    Liangyu Xu\textsuperscript{4}$^{\ast}$, \ Xuemiao Zhang\textsuperscript{1,4}$^{\ast}$, \ 
    Feiyu Duan\textsuperscript{2,4}\thanks{Equal contribution.},  \ 
    \textbf{Sirui Wang\textsuperscript{3,4}\thanks{Corresponding author.}}, \ \\
     \textbf{Rongxiang Weng\textsuperscript{4}}, \ 
    \textbf{Jingang Wang\textsuperscript{4}}, \ 
    \textbf{Xunliang Cai\textsuperscript{4}}
    \\
    \textsuperscript{1} Peking University\quad
    \textsuperscript{2} Beihang University\quad
    \textsuperscript{3} Tsinghua University\quad
    \textsuperscript{4} Meituan \\
    \texttt{zhangxuemiao@pku.edu.cn}\quad 
    \texttt{duanfeiyu@buaa.edu.cn} \quad \\
    \texttt{\{xuliangyu02, wangsirui, wengrongxiang, wangjingang02, caixunliang\}@meituan.com} \\
  } 

\begin{document}
\maketitle
\begin{abstract}
Selecting high-quality data can improve the pretraining efficiency of large language models (LLMs). Existing methods generally rely on heuristic techniques or single quality signals, limiting their ability to evaluate data quality comprehensively. In this work, we propose FIRE, a flexible and scalable framework for integrating multiple data quality raters, which allows for a comprehensive assessment of data quality across various dimensions. FIRE aligns multiple quality signals into a unified space, and integrates diverse data quality raters to provide a comprehensive quality signal for each data point. Further, we introduce a progressive data selection scheme based on FIRE that iteratively refines the selection of high-quality data points. Extensive experiments show that FIRE outperforms other data selection methods and significantly boosts pretrained model performance across a wide range of downstream tasks, {while requiring less than 37.5\% of the training data needed by the \textit{Random} baseline to reach the target performance.}

\end{abstract}

\section{Introduction}
Large language models (LLMs) have demonstrated remarkable performance by utilizing large-scale Transformers to pretrain on trillions of tokens. However, due to the constraints imposed by scaling laws \cite{kaplan2020scaling}, LLMs are quickly nearing their capacity and data limits. As a result, efforts to improve LLM performance have increasingly concentrated on optimizing the quality of pretraining data. 

Numerous studies indicate that effective data selection can significantly enhance the convergence speed and generalization capability of LLMs \cite{engstrom2024dsdm,wettig2024qurating,gao2025metadataconditioningaccelerateslanguage}. Traditional methods predominantly rely on heuristic techniques, such as rule-based filtering \cite{rae2021scaling,raffel2020exploring}, deduplication \cite{abbas2023semdedup,tirumala2024d4}, and assessing proximity to high-quality corpora \cite{xie2023data}. Additionally, some work has focused on improving the evaluation of pretraining data quality by querying authoritative LLMs to determine whether the texts meet specific criteria \cite{wettig2024qurating,sachdeva2024train}. 

\begin{figure}[t]
    \centering
    \includegraphics[width=0.9\linewidth]{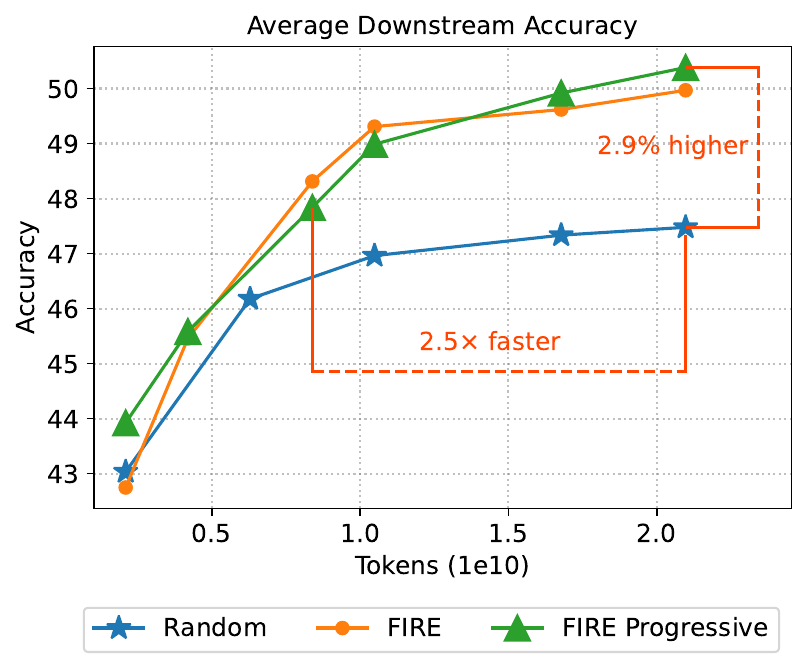}
        \caption{Downstream accuracy with respect to {pretraining} tokens for Random, FIRE, and FIRE Progressive.}
    \label{fig:avg_score_steps}
\end{figure}

Intuitively, assessing the quality of a text involves analyzing it across multiple dimensions. Nevertheless, the methods mentioned above evaluate data quality based on individual aspects, lacking a comprehensive assessment of the data's overall quality. Building on this limitation, the success of querying LLMs \cite{sachdeva2024train} has inspired the straightforward idea of merging various quality rating criteria into a single prompt to collect comprehensive quality signals from authoritative LLMs. However, experimental findings show that this approach considerably weakens the performance of LLMs, as the excessive number of rules makes it challenging for LLMs to follow (further details can be found in Appendix \ref{subsec:Prompt merge effect}). The challenge of adhering to multiple rules underscores the necessity for a more sophisticated strategy to integrate various quality signals effectively.


In this paper, we propose FIRE, a \textbf{F}lexible and scalable framework for \textbf{I}ntegrating multiple data quality \textbf{R}aters, designed to enable \textbf{E}ffective pretraining of LLMs. Initially, we introduce an alignment method to tackle the issue of inconsistent ratings from multiple raters. This method involves ranking the data based on the scores from the original raters and then partitioning it into quantiles. We assess the probability that the data within each quantile is of higher quality (win rate) compared to a reference subset, using this probability as the aligned rating. By fitting a win-rate-quantile curve for each rater, we effectively map the ratings from multiple raters into a unified rating space. Subsequently, to derive a comprehensive signal representing the overall data quality, we integrate the aligned ratings of multiple raters, considering both the intrinsic reliability and orthogonality of the raters. Further, we introduce a progressive data selection scheme based on FIRE that iteratively refines the selection of high-quality data points, balancing computational complexity with the refinement of orthogonality.

Extensive experiments demonstrate that by applying integrated ratings from multiple raters, our method achieves superior results across a variety of downstream tasks. Figure \ref{fig:avg_score_steps} illustrates that FIRE significantly enhances the pretrained model.
We summarize our main contributions as follows:

(1) We propose FIRE, a flexible and scalable framework for integrating multiple data quality raters. FIRE aligns ratings from multiple raters into a unified space and integrates them to provide a comprehensive quality signal for each data point.

(2) We introduce a progressive data selection scheme based on FIRE that iteratively refines the selection of high-quality data points. It achieves a balance between computational complexity and the refinement of orthogonality.

(3) Extensive experiments demonstrate that FIRE enhances the pretrained model's performance by an average of 2.9\%, {while requiring less than 37.5\% of the training data needed by the \textit{Random} baseline to reach the target performance.}

\section{FIRE: Flexible Integration of Quality Ratings}
\label{section:FIRE}
\subsection{Overview of the Method}
We propose FIRE, a method that flexibly integrates multiple raters to comprehensively evaluate data quality. It involves two key processes: \textbf{(a) Rating Alignment} and \textbf{(b) Rater Integration}. Figure \ref{fig:fire} illustrates the overall framework of FIRE. {Many off-the-shelf raters exist in practice. To be integrated by FIRE, a rater must provide a scalar score for each data point and be empirically validated for effectiveness.}

First, we propose an alignment method to map ratings from multiple raters into a unified rating space. Specifically, we involve the probability that the data in each quantile is of higher quality (win rate) compared to a reference subset as the aligned rating. By fitting a win-rate-quantile curve for each rater, we effectively map the ratings from multiple raters into a unified rating space. It's worth noting that the alignment process allows us to quantify the intrinsic reliability of each rater, defined by the win rate of the best data subset selected by the rater relative to the reference subset, thereby reflecting the rater's performance.

We then integrate the aligned ratings of multiple raters, considering both the intrinsic reliability and orthogonality of the raters. We construct an orthogonality graph and calculate centrality through PageRank\cite{page1999pagerank} to quantify the independence among raters. The integrated rating is given by:
\begin{equation}
I(x) = \mathbf{A}(x)^T (\mathbf{o} \odot \bm{\gamma})   
\end{equation}
where \(\mathbf{A}(x)\) is the vector of aligned ratings for data point \(x\), \(\mathbf{o}\) is the overall orthogonality vector, \(\bm{\gamma}\) represents the intrinsic reliability vector of the raters, and \(\odot\) denotes the element-wise product.
\begin{figure}[t]
    \centering
    \includegraphics[width=\linewidth]{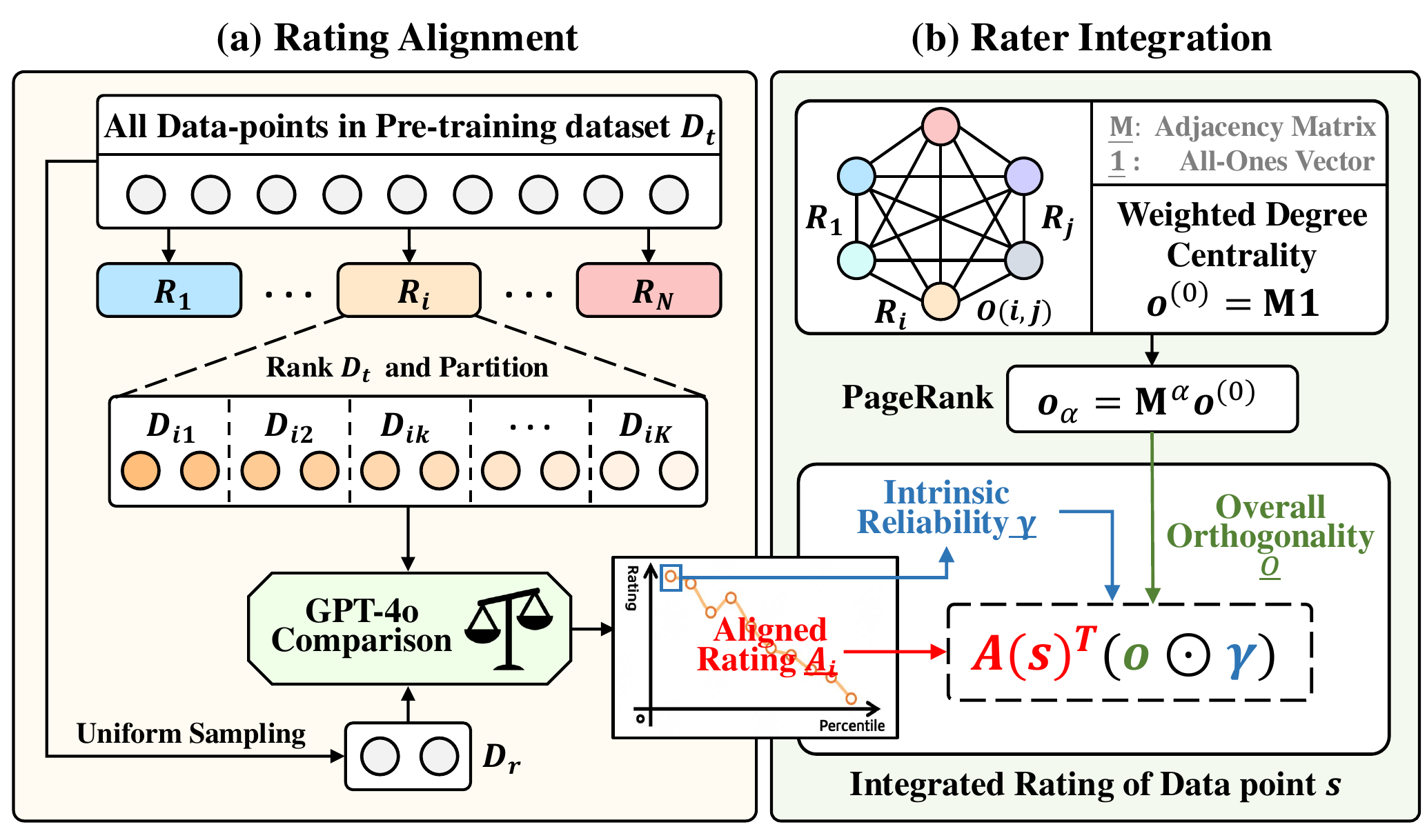}
    \caption{Overall framework of FIRE, which contains two processes: (a) Rating Alignment and (b) Rater Integration.}
    \label{fig:fire}
\end{figure}
\subsection{Rating Alignment}
Aligning ratings from multiple data quality raters is crucial for achieving a credible integrated rater. This process involves standardizing ratings to a consistent scale and eliminating the significant differences in raters' high-quality thresholds, which are the score thresholds that distinguish data contributing positively to pretraining. Appendix \ref{subsec:Analysis of the necessity of rating alignment} offers further analysis on the importance of rating alignment.

Given a pretraining dataset \( \mathcal{D}_{t} \) and multiple raters \( R_1, R_2, \ldots, R_n \), we propose a method to consolidate these ratings into a unified rating space:

\paragraph{Step 1: Sample reference subset.} Uniformly sample a subset \( \mathcal{D}_{r} \) from the pretraining data \( \mathcal{D}_{t} \). {It is worth noting that the data quality distributions of $\mathcal{D}{r}$ and $\mathcal{D}{t}$ are consistent. This consistency allows us to use $\mathcal{D}{r}$ as a representative reference set: if a subset of data is of higher quality than data in $\mathcal{D}{r}$, it can be considered to exceed the typical quality of $\mathcal{D}_{t}$, and vice versa.}

\paragraph{Step 2: Sort and partition data.}Sort \( \mathcal{D}_{t} \) according to \( R_i \), and partition the sorted data into \( k \) intervals.

\paragraph{Step 3: Compare and calculate win rates.} Randomly sample a subset \( \mathcal{D}_{ij} \) from each interval, ensuring that \( |\mathcal{D}_{ij}| = |\mathcal{D}_{r}| \). Use GPT-4o \cite{islam2024gpt} to evaluate how the \( \mathcal{D}_{ij} \) samples impact pretraining in comparison to the reference dataset \( \mathcal{D}_{r} \). Then, calculate the win rate \( w_{ij} \) for each interval \( j \):
\begin{equation}
    w_{ij} = \frac{| \{ x \in \mathcal{D}_{ij} \mid \text{GPT-4o: } x > y, \, y \sim \mathcal{D}_{r} \} |}{| \mathcal{D}_{ij} |}
\end{equation}
where \( \mathcal{D}_{ij} \) signifies the subset sampled from the \( j \)-th rating interval for rater \( R_i \). The win rate \( w_{ij} \) is the proportion of samples \( x \) in \( \mathcal{D}_{ij} \) that GPT-4o determines have higher quality than the comparison sample \( y \) from \( \mathcal{D}_{r} \). {Calculating win rates relative to $\mathcal{D}_{r}$ makes ratings from multiple raters comparable. Moreover, since we sample only 1,000 data points per interval for comparison, the computational cost of this process remains low.} The prompt for GPT-4o is detailed in Appendix \ref{subsec:Prompt for GPT-4o to compare data quality}. {We demonstrate the reliability of using GPT-4o for quality comparison in Appendix \ref{subsec:raog}. }


\paragraph{Step 4: Fitting a win-rate-percentile function.} Employ \( w_{ij} \) as the aligned rating for the midpoint of \( j \)-th rating interval of $R_i$, denoted by (\( p_{j} \), \( w_{ij} \)). Construct a continuous win rate-percentile function from these coordinates using polynomial spline interpolation (detailed in Appendix \ref{subsec:Polynomial spline interpolation for win-rate-percentile function}).

For any data point, we can find the aligned rating from a specific rater by applying the rater's win-rate-percentile function to its original rating and percentile. We apply the alignment method to the ratings of 4 raters on the SlimPajama dataset. The win rates in different percentiles and the fitted functions are provided in Appendix \ref{subsec:Ratings distribution illustration}. It is worth noting that since \(w_{i0}\) represents the win rate of the best data subset selected by rater \(i\) relative to the reference subset, it reflects the performance of rater \(i\) to a certain extent. Therefore, we can use \(w_{i0}\) as the intrinsic reliability of rater \(i\), i.e., \(\gamma_i = W_{i0}\).

\subsection{Rater Integration}
Suppose for raters \( R_1, R_2, \ldots, R_n \), each rater corresponds to a standard basis vector \( \mathbf{v}_1, \mathbf{v}_2, \ldots, \mathbf{v}_n \) in the quality space.  The integrated quality vector \( \mathbf{q}(x) \) of data point \( x \) can be expressed as:
\begin{equation}
\mathbf{q}(x) = \sum_{j=1}^{n} \gamma_j   A_j(x) \mathbf{v}_j    
\end{equation}
where \(\gamma_j\) denotes the intrinsic reliability of rater \(j\), $A_k(x)$ denotes the rating for data point $x$ from rater $R_k$ after alignment.
Ideally, if \( \mathbf{v}_1, \mathbf{v}_2, \ldots, \mathbf{v}_n \) form an orthogonal basis, it is reasonable to measure the overall quality of the data using the L1 norm of \(\mathbf{q}(x)\), as it represents the sum of the scores of data point \(x\) across various orthogonal quality dimensions. However, \( \mathbf{v}_1, \mathbf{v}_2, \ldots, \mathbf{v}_n \) are not necessarily completely independent. There may be raters \( R_i \) and \( R_j \) with a correlation coefficient \( \rho > 0 \) and directly adding the corresponding aligned ratings would increase the weight of a particular quality dimension. To mitigate this issue, we define $O(i,j)$ to quantify the orthogonality of two raters $i$ and $j$ (the formalization of $O(i,j)$ can be found in Appendix \ref{subsec:Formalization of Orthogonality}). For a rater \( R_i \), we apply the sum of its orthogonality with all other raters to weight its rating. If a rater is highly correlated with others, we use a lower orthogonality to penalize. So the integrated rating for data point $x$ can expressed as
\begin{equation} 
I(x) = \sum_{j=1}^{n} \gamma_j o_j A_j(x) 
\label{eq:5} 
\end{equation}
where \( o_j = \sum_{\substack{k=1 \\ k \neq j}}^{n} O(j, k) \) is a quantification of the overall orthogonality \( R_j \) with other raters.

Inspired by Equation (\ref{eq:5}), we find that integration of ratings weighted by orthogonality can be formalized to the centrality problem of graph theory. Formally, we define orthogonality graph of a rater as follows:

\begin{definition}[Orthogonality Graph]\label{def:1}
An orthogonality graph is a complete graph where the vertices \( V_i \) represent the raters \( R_i \). The weight of the edge between two vertices is the orthogonality \( O(i,j) \) between the two raters.
\end{definition}

Based on Definition \ref{def:1}, we provide the following theorem:
\begin{theorem} 
\label{theorem1}
The overall orthogonality $o_i$ of a rater \( R_i \) with other raters can be quantified as \textbf{weighted degree centrality} of the corresponding vertex \( V_i \) in the orthogonality graph.
\end{theorem}
We give Theorem \ref{theorem1}'s proof in Appendix \ref{sec:Proof of theorems}.
Let \( \mathbf{o} = [o_1, o_2, \ldots, o_n]^T \) denote the overall orthogonality vector, where \( o_i \) represents the overall orthogonality of rater \( R_i \). Define \( \mathbf{M} \) as the adjacency matrix of the orthogonality graph, such that \( \mathbf{M}_{ij} = O(i,j) \). Additionally, let \( \mathbf{1} \) be the all-ones vector. We can then derive the following:
\begin{equation}
\mathbf{o}^{(0)} = \mathbf{M} \mathbf{1}
\end{equation}

Considering that in a multi-rater setting, the independence of \( R_i \) might be affected by the orthogonality between different raters \( R_j \) and \( R_k \), we propose an iterative formula, analogous to PageRank \cite{page1999pagerank}:
\begin{equation}
\mathbf{o}^{(k+1)} = d \mathbf{M} \mathbf{o}^{(k)} + (1-d) \mathbf{1}
\end{equation}
where \( d \) is the damping factor and \( k \) denotes the \( k \)-th iteration. {{{In PageRank, node centrality depends on edge weights. Since FIRE defines edge weights via orthogonality, the resulting scores reflect each node’s overall independence.}}} Since the introduction of the damping factor aims to address the rank sinks problem, which is not present in our graph, we find it reasonable to set \( d = 1 \). Assuming the number of iterations is \( \alpha \), and normalizing the final result, our final formula becomes:
\begin{equation}
\label{eq:9}
\mathbf{o}_{\alpha} = \mathbf{M}^{\alpha} \mathbf{o}^{(0)}
\end{equation}
\begin{equation}
\label{eq:10}
\mathbf{o} = \frac{\mathbf{o}_{\alpha}}{\|\mathbf{o}_{\alpha}\|_2}
\end{equation}
where \( \|\cdot\|_2 \) denotes the Euclidean Norm. In this paper we set $\alpha=50$ since $\mathbf{o}_{\alpha}$ tends to stabilize after 50 iterations. The justification for employing Equation (\ref{eq:9}) and (\ref{eq:10}) to quantify the overall orthogonality is provided by Theorem 3 in Appendix \ref{sec:Proof of theorems}.\footnote{When there's a complete correlation among some Raters, the orthogonality drops to zero. In such a case, we don't proceed with the rater integration and consider these multiple raters as one.}

Final version of the integrated rating for data point $x$ can be expressed as:
\begin{equation}
I(x) = \mathbf{A}(x)^T (\mathbf{o} \odot \bm{\gamma})
\label{eq:11}
\end{equation}
where \(\mathbf{A}(x) = [A_1(x), A_2(x), \ldots, A_n(x)]^T\) is the vector of aligned ratings for data point \(x\) from all raters, \(\mathbf{o}\) is the overall orthogonality vector, \(\bm{\gamma} = [\gamma_1, \gamma_2, \ldots, \gamma_n]^T\) represents the intrinsic reliability vector, and \(\odot\) denotes the Hadamard product.

\begin{algorithm}[t]
\caption{\small Progressive Data Selection Scheme}  
\label{alg:progressive_data_selection}
\begin{algorithmic}[1]
\STATE \textbf{Input:} The entire dataset $\mathcal{D}_{t}$; Decay factor $\eta$; Initial number of parts $n$; Part multiplication factor $\beta$; Maximum number of parts $n_{\text{max}}$; Desired dataset size $k$
\STATE \textbf{Output:}  $\mathcal{D}_{s}$: Selected data subset

\STATE Calculate and sort integrated ratings for $\mathcal{D}_{t}$
\STATE Reduce the data to $\eta \%$ of its original size using decay factor $\eta$
\WHILE{size of $\mathcal{D}_{t}$ $\geq k$}
    \STATE Divide the data into $n$ parts based on the quantiles of the integrated ratings
    \FOR{each part $P_i$}
        \STATE Calculate the overall orthogonality in $P_i$
        \STATE Derive refined integrated ratings $\mathcal{S}_{P_i}$
    \ENDFOR
    \STATE Sort the data based on the new integrated ratings $\mathcal{S}_{P_i}$
    \STATE {Select the top $\eta \%$ data according to $\mathcal{S}_{P_i}$}
    \STATE $n \gets \min(n \times \beta, n_{\text{max}})$
\ENDWHILE
\STATE $\mathcal{D}_{s}$ $\gets$ \text{top k elements from} $\mathcal{D}_{t}$
\end{algorithmic}
\end{algorithm}

\begin{table*}[t]
    \centering
    \resizebox{\linewidth}{!}{
    \begin{tabular}{>{\arraybackslash}m{1.7cm} 
    |>{\arraybackslash}m{3.2cm} 
    |>{\centering\arraybackslash}m{1.3cm} 
    >{\centering\arraybackslash}m{1.3cm}
    >{\centering\arraybackslash}m{1cm} 
    >{\centering\arraybackslash}m{1.2cm}
    >{\centering\arraybackslash}m{1cm}
    >{\centering\arraybackslash}m{1.2cm}
    >{\centering\arraybackslash}m{1cm}
    >{\centering\arraybackslash}m{1cm}
    >{\centering\arraybackslash}m{1cm}
    }
    \toprule
        \multicolumn{2}{l|}{\textbf{Method}} & \textbf{ARC-E} & \textbf{ARC-C} & \textbf{SciQ} & \textbf{LogiQA} & \textbf{BoolQ} & \textbf{HellaSw.} & \textbf{PIQA} & \textbf{W.G.} & \textbf{AVG.} \\ 
    \midrule
        \multirow{5}{*}{Baseline} & Random & 48.2 & 22.3 & 84.5 & 19.7 & 60.8 & 32.1 & 63.5 & 49.2 & 47.5 \\
        ~ & DSIR with Book & 36.2 & 19.5 & 73.4 & 21.4 & 61.8 & 29.5 & 62.5 & \textbf{53.6} & 44.7 \\
        ~ & DSIR with Wiki & 37.2 & 18.0 & 76.4 & 23.0 & 58.0 & 27.9 & 57.3 & 51.1 & 43.6 \\
        ~ & Density & 47.2 & 20.0 & 81.7 & 20.3 & 61.5 & 31.4 & \textbf{66.3} & 51.4 & 47.5 \\ 
        ~ & ASK-LLM & 52.6 & 24.8 & 80.2 & 22.1 & \textbf{62.2}& 28.9 & 59.5 & 50.2 & 47.6 \\
    \midrule
        \multirow{4}{*}{\makecell[l]{Baseline\\(1 Rater)}} & QuRating (W.S.) & 47.5 & 21.4 & 81.8 & 21.3 & 61.3 & 31.3 & 62.7 & 52.5 & 47.5 \\ 
        ~ & QuRating (R.E.) & 50.6 & 23.2 & 83.9 & 22.6 & 61.4 & 30.2 & 59.8 & 49.8 & 47.7 \\ 
        ~ & QuRating (F.T.) & 54.1 & 23.0 & 83.5 & 22.0 & 60.9 & 30.4 & 59.5 & 51.7 & 48.1 \\ 
        ~ & QuRating (E.V.) & 50.1 & 21.6 & 84.4 & 20.9 & \textbf{62.2} & 31.9 & 61.2 & 48.8 & 47.6 \\
    \midrule
        \multirow{4}{*}{\makecell[l]{Baseline\\(Integration\\method, \\4 Raters)}} & Comp. Rater & 52.9 & 24.3 & 81.2 & 22.0 & 62.0 & 30.9 & 59.2 & 50.1 & 47.8 \\
        ~ & Max Criteria & 54.1 & 22.6 & 83.3 & 22.7 & 61.3 & 30.8 & 59.8 & 48.9 & 47.9 \\
        ~ & Average & 48.7 & 23.5 & 83.4 & 21.4 & 59.8 & 30.1 & 58.8 & 51.1 & 47.1 \\ 
        ~ & Mix Criteria\(^\dagger\) & 49.6 & 22.1 & 83.6 & 25.7 & 61.8 & 29.7 & 58.6 & 50.4 & 47.7 \\
    \midrule
        \multirow{5}{*}{\makecell[l]{Raters\\Integration}} & FIRE (2 Raters) & 55.4 & 24.9 & 83.6 & 21.3 & 60.0 & 31.6 & 60.1 & 50.4 & 48.4 \\
        ~ & FIRE (3 Raters) & 58.4 & 25.7 & 85.1 & \textbf{23.1} & 59.8 & 32.3 & 61.2 & 51.2 & 49.6 \\ 
        ~ & FIRE (4 Raters) & 59.1 & 26.4 & 86.0 & 21.0 & 61.8 & 32.9 & 59.7 & 52.8 & 50.0 \\
    \cmidrule{2-11}
        ~ & \makecell[l]{FIRE (4 Raters) Prg.} & \textbf{59.2} & \textbf{27.0} & \textbf{86.9} & 23.0 & 60.2 & \textbf{33.0} & 62.4 & 51.6 & \textbf{50.4} \\
    \bottomrule
    \end{tabular}
    }
    \caption{Downstream tasks results for different rating integration method. We report accuracy for each task, and the best performances are marked in bold. For rater integration, we report the average score of all the combinations. Detailed results can be found in the Appendix \ref{subsec:Results of different combinations}. Abbreviations: HellaSw. = HellaSwag, W.G. = WinoGrande, AVG. = Average, W.S. = Writing Style, R.E. = Required Expertise, F.T. = Facts and Trivia, E.V. = Educational Value, Prg = Progressive, Comp. = Comprehensive. \(\dagger\): We implement the method from QuRating\cite{wettig2024qurating}.}
    \label{main results1}
\end{table*}

\section{Progressive Data Selection via FIRE}
The most intuitive data selection method involves ranking the integrated ratings based on FIRE for the pretraining dataset, and then selecting the top $k$ highest-rated data points. Nonetheless, our analysis shows that after ranking $\mathcal{D}_{t}$ based on the integrated ratings, there is a change in the overall orthogonality $\mathbf{o}$ calculated from data subsets in different quantiles (Figure \ref{ortho}). The phenomenon arises because the top data better reflects the quality emphasized by the raters, while the tail data often contains more noise and low-quality information, leading to changed orthogonality among the raters. To refine the data selection process and mitigate the coarseness introduced by computing orthogonality on the entire dataset, we propose a progressive data selection scheme based on FIRE.

Specifically, as shown in Algorithm \ref{alg:progressive_data_selection}, we first calculate the integrated ratings for the data points in $\mathcal{D}_{t}$ based on FIRE, then sort the data points, but we don't select them right away based on these ratings. {It is reduced to $\eta\%$ of its original size by selecting the top $\eta\%$ of data based on integrated ratings.} Then, the data is partitioned into $n$ segments based on integrated ratings' quantiles. Orthogonality is computed within each segment to determine refined integrated ratings. After sorting the data according to new ratings, it's further reduced to $\eta \%$ of its initial size. The number of segments is increased by a factor of $\beta$, unless it reaches the maximum threshold $n_{\text{max}}$, in which case the data is divided into $n_{\text{max}}$ segments. The iterative process of calculating orthogonality within progressively smaller subsets continues before the subsequent reduction leaves less than $k$ data points for selection.

\section{Experiments}

\subsection{Experimental setup}

\paragraph{Setup} 
We use SlimPajama \cite{cerebras2023slimpajama} as the selection pool, with a total scale of 627B. And we employ the Llama \cite{touvron2023llama} tokenizer to divide the entire dataset into sequences of length 1024. During the data selection process, we select the top portion of data with the highest rating. For integrating different raters, we carry out experiments based on the four single raters of QuRating \cite{wettig2024qurating}: Writing Style, Facts and Trivia, Educational Value, and Required Expertise. For progressive data selection (FIRE Progressive), we set \(\eta=60, \beta=20\). For model training, we train a model with 1.3B parameters for 10,000 steps (equivalent to 20B tokens) and a 3B model for 200B tokens, with bfloat16 format during both training and testing. 

\paragraph{Evaluation} 
We utilize lm-evaluation-harness \cite{gao2021framework} to assess the models' performance across eight downstream tasks: ARC-E \cite{clark2018think}, ARC-C \cite{clark2018think}, SciQ \cite{welbl2017crowdsourcing}, LogiQA \cite{liu2020logiqa}, BoolQ \cite{clark2019boolq}, HellaSwag \cite{zellers2019hellaswag}, PIQA \cite{bisk2020piqa}, and WinoGrande \cite{sakaguchi2021winogrande}. 
We employ in-context learning for the evaluation, selecting enough examples to fill the window length of 1024 tokens for each task. Standard accuracy metrics are reported for all tasks.

\paragraph{Baselines} 
In addition to comparing FIRE with four single raters from Qurating, we also compare it with the following methods: (1) \textit{Random}: randomly selecting data from the original training corpus. (2) \textit{DSIR} \cite{xie2023data}: utilizing importance sampling for data selection, and we chose Wikipedia and Books as target domains. (3) \textit{Density} \cite{sachdeva2024train}: using KDE to estimate data density in the training corpus and employing inverse sampling. (4) \textit{ASK-LLM} \cite{sachdeva2024train}: using a comprehensive prompt to label high-quality data, and train a T5-based classifier.

Furthermore, we compare several basic rating integration methods: (1) \textit{Comprehensive Rater}: integrating multiple single rater criteria into a single prompt to obtain annotations from GPT-4o, then training a comprehensive quality rater. (2) \textit{Max Criteria}: aligning ratings and selecting the highest value in each dimension as the final rating. (3) \textit{Average}: arithmetic mean integration of normalized ratings from each rater. (4) \textit{Mix Criteria}: we follow QuRating\cite{wettig2024qurating} to merge and deduplicate the top data selected by each single rater, followed by random sampling. These four methods are applied to the integration of four raters. More details can be found in Appendix \ref{subsec:Integration Baselines}.


\subsection{Main Results}

Table \ref{main results1} show our main results.   We find that:

\paragraph{FIRE is superior to other integration methods.} FIRE demonstrates greater effectiveness than existing integration methods, {while introducing minimal additional computational cost (see Appendix \ref{appendix:flops} for a detailed analysis of computational cost).} As shown in Table \ref{main results1}, \textit{Comprehensive Rater} with a multi-dimension prompt results in an average score that is even lower than the single-dimension rater \textit{Facts and Trivia}. This suggests that GPT-4o still falls short in assessing data quality from a broad perspective. Both \textit{Mix Criteria} and \textit{Max Criteria} are inferior to FIRE, indicating that a comprehensive evaluation of FIRE is more beneficial. \textit{Average} simply calculates the mean of all ratings and the experimental results of FIRE demonstrate an improvement over \textit{Average}. {To demonstrate the robustness of FIRE, we conduct additional experiments integrating other raters. Detailed results are provided in Appendix \ref{subsec:ior}.}

\paragraph{FIRE outperforms the single raters and other data selection methods.} Comparing FIRE to the individual raters, it demonstrates significant improvements, with an average score increase of up to 1.9\% over the best single rater and 2.9\% over random selection. {From a data efficiency perspective, FIRE achieves comparable performance using less than 37.5\% of the data required by the \textit{Random} baseline.} Additionally, our method outperforms \textit{QuRating} and other data selection methods, validating the high quality of the data selected by FIRE.

\paragraph{Adding more raters can lead to better performance.} We observe that as the number of integrated raters increases, the overall effect gradually improves. The average score of FIRE with three raters surpasses that of the integration of two raters, and further increases when four raters are integrated. This indicates that our rater integration method is scalable: incorporating a broader range of raters not only provides a more comprehensive evaluation of the samples but also allows for a better understanding of the importance of each metric.


\paragraph{Progressive selection further improves FIRE} After integrating progressive selection, we observe a notable improvement in the model's performance on downstream tasks. Compared to FIRE (4 Raters), the average score of FIRE (4 Raters) Prg. increases by 0.4\%. The most significant improvement is seen in PIQA, with an absolute score boost of 2.7\%. These results validate the effectiveness of the progressive selection method in choosing high-quality data.

\subsection{Analysis}



\begin{figure}[t]
    \centering
    \includegraphics[width=\linewidth]{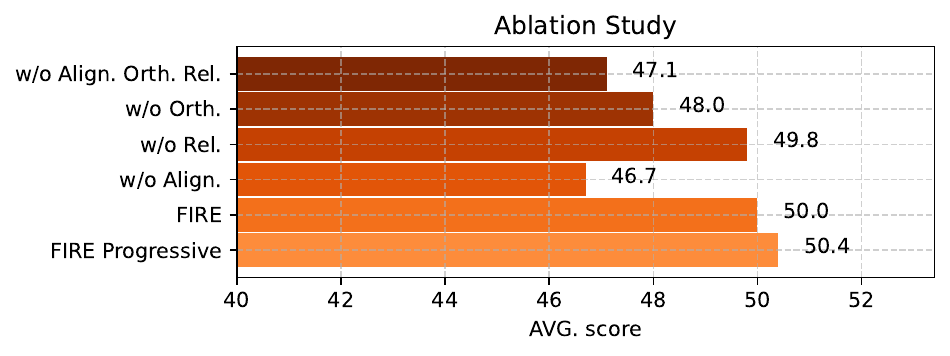}
    \caption{Ablation experiments on the impact of different rating integration strategies in FIRE.}
    \label{ablation study}
\end{figure}

\paragraph{Ablation Study} 
We integrate four single raters, and subsequently remove \textit{Rating Alignment (Align.)}, \textit{Intrinsic Reliability (Rel.)}, and \textit{Orthogonality (Orth.)}, as well as remove all (directly averaging on the rating post-normalization), in order to investigate the impact of each component in FIRE.





From Figure \ref{ablation study}, we can find that: (1) Rating alignment is a crucial step. We note that without rating alignment, the score drops by 
3.3\%, even falling below the direct average. As we previously detailed, the alignment allows for better comparability between the ratings from different raters, making their integration more reasonable.
(2) Both orthogonality and intrinsic reliability can further enhance model performance, with the impact of orthogonality being relatively significant (a drop of 2\% w/o Orth.), while the improvement from Intrinsic Reliability is rather subtle (a drop of 0.2\% w/o Rel.)
(3) The combination of all components yields the best results. This implies that these methods of integrating the ratings are complementary. By superimposing both methods, we can achieve a rating that more accurately reflects the actual contribution of the sample to the pretraining.

\paragraph{Training Efficiency}

\begin{figure}[t]
\centering


\begin{subfigure}[b]{0.23\textwidth}
    \includegraphics[width=\textwidth]{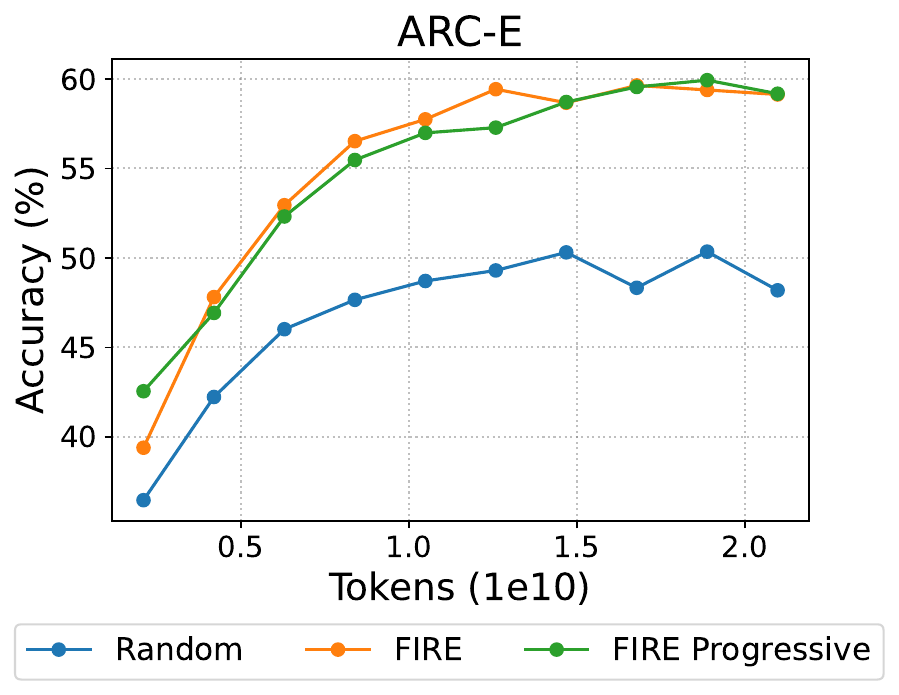}
    \subcaption{ARC-E}
\end{subfigure}
\begin{subfigure}[b]{0.23\textwidth}
    \includegraphics[width=\textwidth]{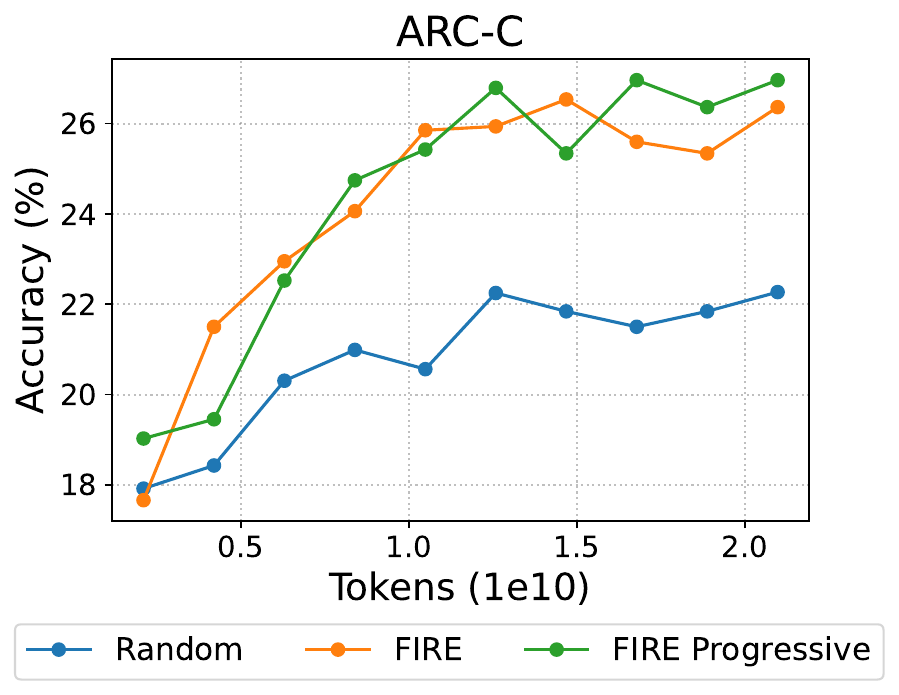}
    \subcaption{ARC-C}
\end{subfigure}
\begin{subfigure}[b]{0.23\textwidth}
    \includegraphics[width=\textwidth]{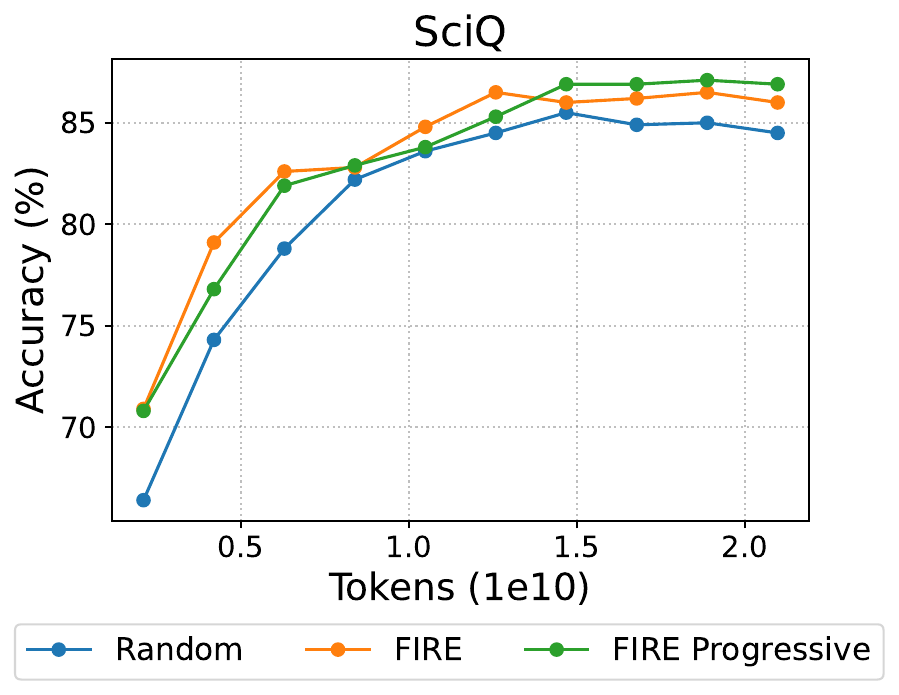}
    \subcaption{SciQ}
\end{subfigure}
\begin{subfigure}[b]{0.23\textwidth}
    \includegraphics[width=\textwidth]{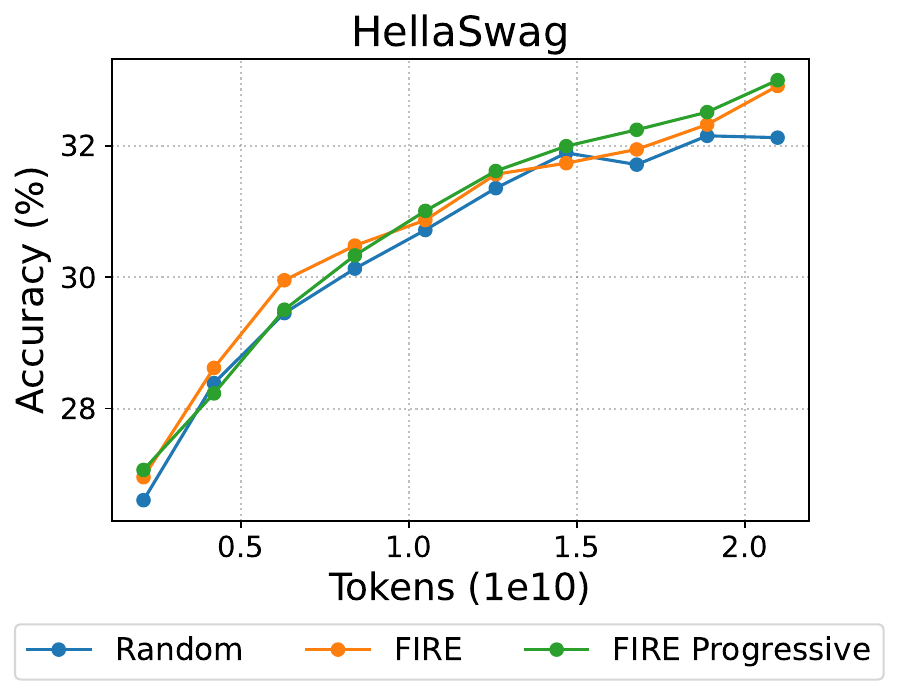}
    \subcaption{HellaSwag}
\end{subfigure}

\caption{The in-context learning results with respect to pretraining tokens on four downstream tasks: ARC-E, ARC-C, SciQ, and HellaSwag.}
\label{fig:score_steps}
\end{figure}

Figures \ref{fig:avg_score_steps} and \ref{fig:score_steps} show how the model's performance on downstream tasks evolves with the pretraining tokens. In terms of average score, our method outperforms the random baseline by 2.9\%. From the training efficiency perspective, our method reduces training tokens to achieve a certain performance level by more than half.
In addition, our method shows a significant advantage in the ARC-E/C and SciQ tasks, consistently scoring higher than the random baseline. However, on the HellaSwag task, our method's performance is similar to the random baseline and does not consistently surpass it. One possible explanation is that HellaSwag is an especially challenging dataset, which makes it difficult to discern performance differences on the 1.3B scale model.

\paragraph{Larger datasets and models}

\begin{table}[t]
    \centering
    \scalebox{0.8}{
    \begin{tabular}{l|c|ccc}
    \toprule
        Method & Pretraining FLOPS & ARC-C & HellaSwag & AVG. \\ 
    \midrule
        \multicolumn{5}{l}{Model Size = 1.3B} \\
    \midrule
        Random & $32.2 \times 10^{19}$ & 23.6 & \textbf{34.4} & 48.6 \\ 
        FIRE & $16.1 \times 10^{19}$ & \textbf{26.4} & 32.9 & \textbf{50.0} \\ 
    \midrule
        \multicolumn{5}{l}{Model Size = 3B} \\
    \midrule
        Random & $377.5 \times 10^{19}$ & 26.8 & 49.4 & 54.1 \\ 
        FIRE & $377.5 \times 10^{19}$ & \textbf{28.8} & \textbf{51.7} & \textbf{55.7} \\ 
    \bottomrule
    \end{tabular}
    }
    \caption{Experimental results on larger models and datasets.}
    \label{3b_model}
\end{table}

To validate our method on larger models and datasets, we conduct several additional experiments: (1) Training a 1.3B parameter model for 40B tokens using randomly sampled data; (2) Training a 3B parameter model for 200B tokens for both random sampling and FIRE (four raters integration). As illustrated in Table \ref{3b_model}, the results indicate that the FIRE method outperforms \textit{Random} with fewer training FLOPS in the 1.3B parameter model setting. Furthermore, in the 3B parameter model setting, FIRE exceeds \textit{Random} by an average of 1.6\%, demonstrating the robustness and scalability of our method with larger models and training datasets.

\paragraph{Ablation Study for Progressive Selection}
\begin{figure}[t]
    \centering
    \scalebox{0.8}{
    \includegraphics[width=0.9\linewidth]{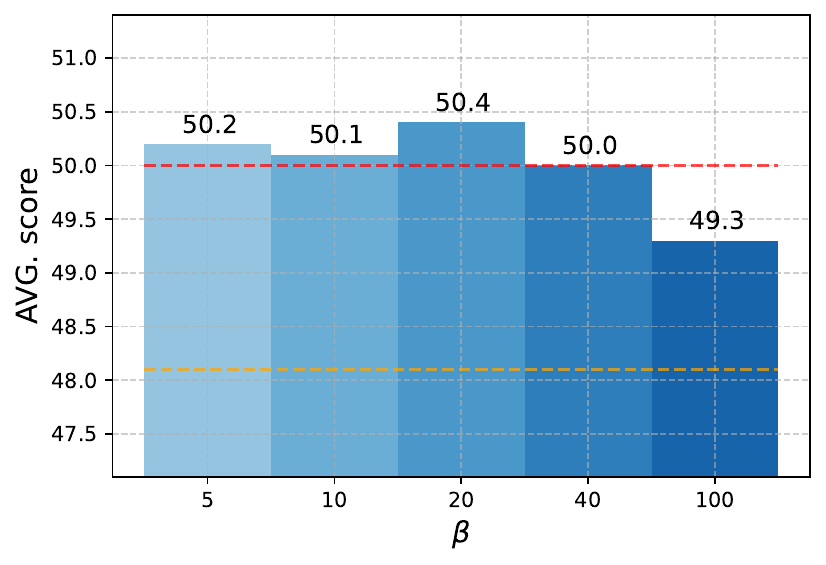}
    }
    \caption{The impact of the partition multiplier factor $\beta$ on FIRE Progressive performance. The red and orange dashed lines respectively represent the scores of FIRE and Random.}
    \label{fig:beta ablation}
\end{figure}

We conduct an ablation study on the partition multiplier factor $\beta$ for the FIRE Progressive approach, with the outcomes shown in Figure \ref{fig:beta ablation}. The results show that for a majority of $\beta$ values, FIRE Progressive scores surpass those of FIRE and Random, suggesting that the progressive selection method contributes to a consistent enhancement of the FIRE framework.

\paragraph{Case study}





\begin{figure}[t]
\centering


\begin{subfigure}[b]{0.24\textwidth}
    \includegraphics[width=\textwidth]{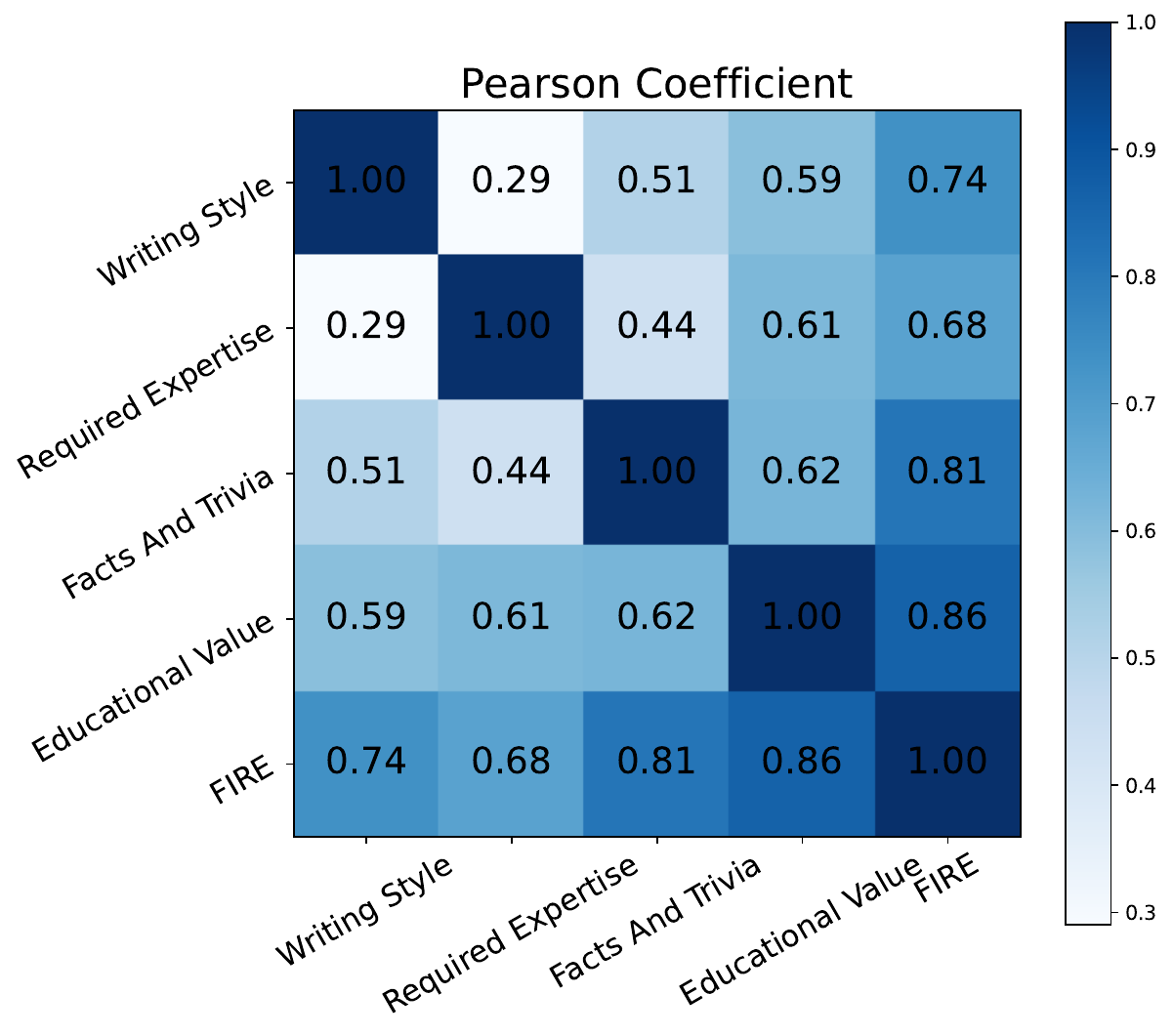}
    \subcaption{}
    \label{fig:rater_corr}
\end{subfigure}
\begin{subfigure}[b]{0.23\textwidth}
    \includegraphics[width=\textwidth]{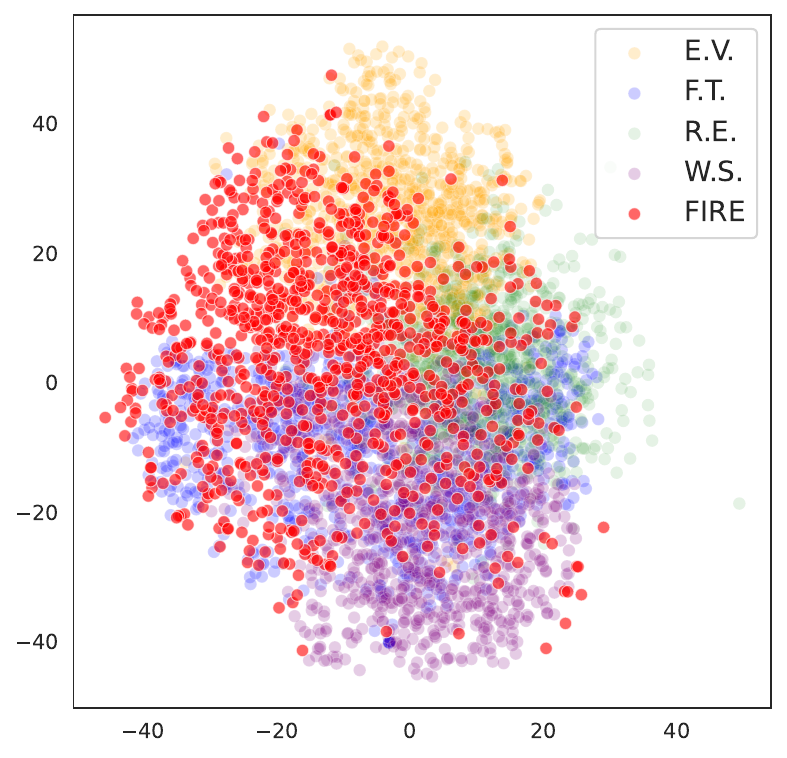}
    \subcaption{}
    \label{fig:diversity}
\end{subfigure}

\caption{(a) The Pearson correlation between different raters. (b) Illustration of top samples rated by each method. We use Sentence-T5 to encode texts, and employ t-SNE to perform dimensionality reduction.}
\label{fig:rater_corr and diversity}
\end{figure}



\begin{table}[t]
    \centering
    \scalebox{0.85}{
    \begin{tabular}{l|ccccc}
    \toprule
        Rater/Dimension & W.S. & R.E. & F.T. & E.V. & Comp. \\ 
    \midrule
        W.S. & \textbf{94.0} & 53.5 & 62.9 & 25.6 & 84.1 \\ 
        R.E. & 39.0 & 85.8 & 93.8 & \textbf{97.1} & 98.0 \\ 
        F.T. & 46.0 & \textbf{97.3} & 88.7 & 42.2 & 98.3 \\ 
        E.V. & 50.8 & 76.3 & \underline{96.5} & 47.1 & 99.1 \\ 
        FIRE & \underline{58.2} & \underline{90.5} & \textbf{97.3} & \underline{54.1} & \textbf{99.5} \\ 
    \bottomrule
    \end{tabular}
    }
    \caption{The percentage of high-quality data across various dimensions, for top data selected by each rater. Underlined scores indicate the second highest. Comp. = Comprehensive.}
    \label{fig:multi-dim}
\end{table}

We extract 1M samples from the corpus and compute the pairwise Pearson correlation of the ratings across all dimensions. As illustrated in Figure \ref{fig:rater_corr}, the FIRE rating exhibits a strong correlation with all other ratings, confirming the effectiveness of the FIRE framework in consolidating ratings across multiple dimensions. By employing the FIRE framework, we can effectively select data that exhibits high ratings across all dimensions.

Moreover, we analyze the data properties that FIRE focuses on compared to other raters. Based on the four dimensions in QuRating and a comprehensive dimension, we determine the percentage of high-quality data selected by each rater in each dimension. We pick 1000 data points at random from the top 20B data that are selected by each method. Then we use GPT-4o to assess each dimension six times, taking the average as the final evaluation result. Refer to the Appendix \ref{subsec:Multi-dimension analysis prompt} for more details. As shown in Table \ref{fig:multi-dim}: (1) From a comprehensive perspective, FIRE shows the best results, as evidenced by our performance in downstream tasks. (2) In terms of each dimension, FIRE consistently achieves relatively high accuracy, demonstrating that the data selected by this method maintains high quality across all dimensions.

To assess whether FIRE selects more diverse samples than single raters, we extract the top 1000 texts from each method and encode them using Sentence-T5 \cite{ni2022sentence}. We then apply t-SNE \cite{van2008visualizing} for dimensionality reduction and visualize the results in Figure~\ref{fig:diversity}. In the latent semantic space, samples selected by single raters appear more clustered, while those selected by FIRE are more broadly distributed, indicating higher diversity. Combined with prior case studies showing high quality, FIRE demonstrates a strong balance between diversity and quality.

\section{Related Works}
When pretraining language models, a large amount of text corpus is often crawled from the internet. However, several studies \cite{li2023textbooks, zhou2024lima,Duan_Zhang_Wang_Que_Liu_Rong_Cai_2025} suggest that high-quality data is more beneficial to the model's performance. To select high-quality data, a common strategy involves utilizing rules crafted by experts \cite{raffel2020exploring, rae2021scaling, laurenccon2022bigscience, together2023redpajama, penedo2024fineweb} and removing duplicate sentences \cite{lee2022deduplicating, sorscher2022beyond, abbas2023semdedup, cerebras2023slimpajama, tirumala2024d4}. However, they often fall short in effectively selecting high-quality data based on semantic content. An alternative approach involves utilizing a target data source or proxy model \cite{wenzek2020ccnet, xie2023data, marion2023less, thakkar2023self, engstrom2024dsdm, yu2024mates}.

Training a classifier is a more straightforward method \cite{du2022glam, gururangan2022whose, zhang2024autonomous, wettig2024qurating, sachdeva2024train}. \citeauthor{du2022glam} (\citeyear{du2022glam}) implemented a logistic regression binary classifier to score the data, while some studies train more complex scorers 
 \cite{zhang2024autonomous, sachdeva2024train}. Additionally, QuRating \cite{wettig2024qurating} trains multiple raters with a finer-grained approach to analyze the contribution of data to model performance improvement from different dimensions. Other studies\cite{zhang2025preferencecurriculumllmspretrained,zhang2025frameboostingllmsfourquadrant} explore methods to boost pretraining efficiency by curriculum learning.
 
Prior methods mainly select data from a single perspective. Although QuRating introduces multidimensional raters, it does not systematically address their integration—a challenge our work explicitly tackles.

\section{Conclusion}
We propose FIRE, a flexible and scalable framework that integrates multiple data quality raters for comprehensive, multi-dimensional data assessment. First, ratings from different dimensions are aligned into a unified space. Then, orthogonality is introduced to adjust rater weights. To handle orthogonality variations across rating ranks, we adopt a progressive approach for fine-grained data selection. Experiments on the SlimPajama dataset show that FIRE outperforms other selection methods, substantially improving pretrained model performance across diverse downstream tasks.

\section{Limitations and Future Works}

\paragraph{Linear assumption for orthogonality integration} We hypothesize our integration on a linear additive relation. While this assumption simplifies the computations, it might limit our ability to capture complex interactions between different dimensions. Future research could explore incorporating non-linear systems to adjust rater weights, potentially boosting performance.

\paragraph{Number of raters} We've tested our method with four different raters and the results have been promising. To make our integration method more robust and reliable, future tests could include more raters from diverse dimensions, which would ultimately help us build a more resilient and versatile rating integration system.


\bibliography{main}
\newpage
\appendix

\section{Ethical Considerations}
Training large language models (LLMs) demands a substantial amount of electrical power, resulting in significant carbon emissions. To address this issue, we aim to develop efficient data selection methods that reduce the computational resources required for model training, thereby mitigating environmental impact. Furthermore, by meticulously curating high-quality data, we can enhance model performance and minimize the occurrence of hallucinations. This not only improves the reliability of the models but also helps curb the spread of fake news and misinformation, addressing critical societal concerns.

\section{Proof of theorems}
\label{sec:Proof of theorems}
In this section, we present the proofs for the two previously mentioned theorems. Theorem 2 facilitates the transformation of orthogonality calculations into a centrality problem within a graph. Meanwhile, Theorem 3 rigorously demonstrates the convergence of our framework. Specifically, it establishes that after several iterations, the vector $\frac{\mathbf{o}}{||\mathbf{o}||_2}$ will assuredly converge to a fixed vector, thus precluding divergence.

\newtheorem*{repTheorem}{Theorem 2}

\begin{repTheorem}
The overall orthogonality $o_i$ of a rater \( R_i \) with other raters can be quantified as \textbf{weighted degree centrality} of the corresponding vertex \( V_i \) in the orthogonality graph.
\end{repTheorem}

\begin{proof}
Consider that the weighted degree centrality of vertex \( V_i \) is the sum of the weights of the edges connecting \( V_i \) to all vertices in the set of adjacent vertices \( A_i \). Since the graph of orthogonality is a complete graph, we have
\[
    C(V_i) = \sum_{V_j \in A_i} O(i,j) = \sum_{\substack{j=1 \\ j \ne i}}^{n} O(i,j)
\]

This is consistent with the definition of the overall orthogonality in Equation (4) of the main text.
\end{proof}

\newtheorem*{reprepTheorem}{Theorem 3}

\begin{reprepTheorem}
As \(\alpha \to +\infty\), \(\mathbf{o}\) will eventually converge to a fixed unit vector.
\end{reprepTheorem}

\begin{proof}
Since the Graph of orthogonality is an undirected graph, the adjacency matrix \(\mathbf{M}\) is a symmetric matrix. According to the Spectral Theorem for Symmetric Matrices, \(\mathbf{M}\) can be diagonalized, and all corresponding eigenvectors can form an orthogonal basis.

Since for any \(i, j\), \(O(i, j) \geq 0\), we can deduce that \(\mathbf{M}\) is a non-negative matrix. Additionally, since the Orthogonality Graph is a complete graph, \(\mathbf{M}\) is an irreducible matrix. By the Perron-Frobenius Theorem, we obtain that:

\[
\begin{aligned}
& \exists \lambda \in \mathbb{R}, \lambda > 0 \\
\text{s.t.} \ & \lambda = \max{\{\mu \mid \mu \in \sigma(\mathbf{M})\}},
\end{aligned}
\]
where $\sigma(\mathbf{M})$ denotes the set of eigenvalues of $\mathbf{M}$.
Assume \(\mathbf{M}\) has \(m\) eigenvalues \(\lambda_1, \lambda_2, ..., \lambda_m\) arranged in descending order, where \(\lambda_1>0\). Each eigenvalue \(\lambda_i\) corresponds to the eigenvectors \(\mathbf{v}_{i1}, ..., \mathbf{v}_{ip_i}\), where \(p_i\) is the algebraic multiplicity of \(\lambda_i\).
We can decompose \(\mathbf{o}^{(0)}\) into each eigenvector
\[
    \mathbf{o}^{(0)} = \sum^{m}_{i=1}{\sum^{p_i}_{j=1}{c_{ij}\mathbf{v}_{ij}}}
\]

Therefore, we have
\[
\begin{aligned}
    \mathbf{M}^{\alpha}\mathbf{o}^{(0)} &= \mathbf{M}^{\alpha}\sum^{m}_{i=1}{\sum^{p_i}_{j=1}{c_{ij}\mathbf{v}_{ij}}} \\
&=\sum^{m}_{i=1}{\sum^{p_i}_{j=1}{c_{ij}\mathbf{M}^{\alpha}\mathbf{v}_{ij}}} \\
&=\sum^{m}_{i=1}{\sum^{p_i}_{j=1}{c_{ij}{\lambda_i}^{\alpha}\mathbf{v}_{ij}}} \\
&={\lambda_1}^{\alpha}\sum^{m}_{i=1}{\sum^{p_i}_{j=1}{c_{ij}(\frac{\lambda_i}{\lambda_1})^{\alpha}\mathbf{v}_{ij}}} \\
\end{aligned}
\]

Given that \(\forall i \neq 1\), \(|\frac{\lambda_i}{\lambda_1}| < 1\), thus 
\[
\lim_{\alpha \to +\infty}{\left(\frac{\lambda_i}{\lambda_1}\right)^{\alpha}} = 0
\]

We obtain
\[
\mathbf{o} = \lim_{\alpha \to +\infty}{\mathbf{M}^{\alpha}\mathbf{o}^{(0)}} = {\lambda_1}^{\alpha}\sum^{p_1}_{j=1}{c_{1j}\mathbf{v}_{1j}} \\
\]

Thus
\[
\begin{aligned}
    \frac{\mathbf{o}}{||\mathbf{o}||_2} &= \frac{{\lambda_1}^{\alpha}\sum^{p_1}_{j=1}{c_{1j}\mathbf{v}_{1j}}}{||{\lambda_1}^{\alpha}\sum^{p_1}_{j=1}{c_{1j}\mathbf{v}_{1j}}||_2} \\
    &= \frac{\sum^{p_1}_{j=1}{c_{1j}\mathbf{v}_{1j}}}{||\sum^{p_1}_{j=1}{c_{1j}\mathbf{v}_{1j}}||_2}
\end{aligned}
\]
The right-hand side of the formula is a fixed unit vector. Therefore, the theorem is proven.
\end{proof}

\section{FIRE Analysis}
\subsection{Analysis of the necessity of rating alignment}
\label{subsec:Analysis of the necessity of rating alignment}
Multiple raters may employ different scales and criteria for assessing data quality, which can cause substantial problems if their ratings are integrated without appropriate standardization. For example, some raters may prioritize grammatical accuracy using a numerical scale, while others might assess semantic relevance using a percentage scale. Moreover, the rating thresholds distinguishing data that positively contribute to pretraining can significantly differ among raters. Without rating alignment, the integrated ratings can be misleading, inaccurately reflecting the actual quality of the data point. The subsequent examples and analyses underscore the importance of rating alignment for a reasonable data quality evaluation:
\begin{itemize}
    \item Different Scales. Suppose we have Rater A, who assesses data quality on a 1 to 10 scale based on grammatical accuracy, and Rater B, who evaluates semantic relevance on a 0\% to 100\% scale. Let's say a particular data point receives an 8 from Rater A and 85\% from Rater B. If we naively average these ratings, we get: \( \frac{8 + 85}{2} = 46.5 \). This score does not genuinely reflect the data quality as the scales used are inherently different. Hence, it is crucial to standardize the ratings onto a common scale to facilitate meaningful comparisons.
    \item Different Quality Thresholds. Even with ratings standardized to a common scale, we face the problem of varying quality thresholds distinguishing data that positively contribute to pretraining. For example, Rater A may deem ratings above 5 as high-quality, whereas Rater B may view ratings above 80\% as high-quality. Suppose we standardize both ratings to a 0-1 scale, turning an 8 from Rater A into 0.8 and 85\% from Rater B into 0.85. Despite this standardization, direct comparison of the two raters' scores remains impractical due to their differing threshold values for differentiating data quality.
\end{itemize}

\subsection{Prompt for GPT-4o to compare data quality}
\label{subsec:Prompt for GPT-4o to compare data quality}
\begin{tcolorbox}[colback=white!95!gray,colframe=gray!50!black,rounded corners,label={prompt}, title={Prompt for GPT-4o to compare data quality}]
\begin{lstlisting}[breaklines=true, xleftmargin=0pt, breakindent=0pt, columns=fullflexible, mathescape, numbers=none]
Compare two text excerpts and choose the text which contain more informative signal for pretraining a large-language model.

An informative datapoint should be well-formatted, contain some usable knowledge of the world, and strictly NOT have any harmful, racist, sexist, etc. content. Aspects that should NOT influence your judgement:
1. The length of the text
2. The order in which the texts are presented

Note that the texts are cut off, so you have to infer their contexts. The texts might have similar quality, but you should still make a relative judgement and choose the label of the preferred text. 

[Option A]
... {text a} ...
[Option B]
... {text b} ...

Now you have to choose between either A or B. Respond only with a single word.
\end{lstlisting}
\end{tcolorbox}

\subsection{Reliability Analysis of GPT-4o}
\label{subsec:raog}
{QuRating\cite{wettig2024qurating} points out that GPT is more effective at comparing the relative quality between two data samples than performing absolute quality evaluation. To further assess the reliability of GPT-4o, we use QuRating (Educational Value) as a case study and conduct a win-rate evaluation involving human experts. Specifically, we compare the win rates assigned by human annotators and GPT-4o across different rating percentiles, with results presented in descending order of percentile rating in Table~\ref{tab:gpt4o_reliability}. The two sets of win rates exhibit a Pearson correlation of 0.99, indicating strong agreement and suggesting that GPT-4o does not introduce significant bias in the annotation process. This high consistency supports the reliability of using GPT-4o for quality assessment. Furthermore, GPT-4o is used only to estimate win rates between rater-selected samples and random subsets, rather than to assign absolute scores, which further reduces bias.}
\begin{table*}[htbp]
\centering
\begin{tabular}{lcccccccccc}
\toprule
Percentile & 10\% & 20\% & 30\% & 40\% & 50\% & 60\% & 70\% & 80\% & 90\% & 100\% \\
\midrule
GPT-4o       & 0.773 & 0.705 & 0.625 & 0.600 & 0.545 & 0.513 & 0.480 & 0.425 & 0.340 & 0.273 \\
Human Expert & 0.797 & 0.698 & 0.657 & 0.616 & 0.533 & 0.501 & 0.469 & 0.433 & 0.317 & 0.300 \\
\bottomrule
\end{tabular}
\caption{Win-rate comparison between GPT-4o and human experts for QuRating (Educational Value) across different rating percentiles.}
\label{tab:gpt4o_reliability}
\end{table*}

\subsection{Polynomial spline interpolation for win-rate-percentile function}
\label{subsec:Polynomial spline interpolation for win-rate-percentile function}
To obtain the aligned rating for each data point, we consider the win rate \( w_{ij} \) of rater \( i \) in the \( j \)-th rating interval as the rating of the midpoint of that interval, denoted by coordinates (\( p_{j} \), \( w_{ij} \)). We then complete the rater's win-rate-percentile function using polynomial spline interpolation to derive a continuous and smooth win-rate-percentile function. The polynomial spline interpolation function \( S(p) \) is defined as follows:

\begin{equation}
\begin{aligned}
S(p) =& a_k (p - p_k)^n + b_k (p - p_k)^{n-1} + \cdots  \\ &+ y_k (p - p_k)^2 + z_k (p - p_k) + d_k, \\
&\qquad \qquad \qquad \qquad p_k \leq p < p_{k+1}
\end{aligned}
\end{equation}
where \( p \) denotes the percentile, \( p_k \) and \( p_{k+1} \) are the boundaries of the \( k \)-th interval, \( n \) is the degree of the polynomial, and \( a_k \), \( b_k \), \(\cdots\), \( y_k \), \( z_k \), and \( d_k \) are the coefficients determined through the spline interpolation process. 

\subsection{Ratings distribution illustration}
\label{subsec:Ratings distribution illustration}
Figure \ref{fitting} shows the win rates of samples in different percentile intervals and the fitted Rating-Percentile curves for 4 raters. 
The original ratings provided by the four raters exhibit significant differences, as illustrated in Figure 4 of QuRating\cite{wettig2024qurating}. The alignment introduces a random subset for comparison, which makes the ratings from different raters comparable, mapping the ratings into a similar range. This forms the foundation for the subsequent weighted integration of the raters, which explains the poor performance without alignment. However, even after alignment, there are still noticeable differences in the rating distributions of different raters. For instance, in Figure \ref{fitting:WS} and \ref{fitting:FT}, there are significant differences in the win rate of the first quartile and the middle section of the curve.

\begin{figure}[htbp]
    \centering
    \begin{subfigure}[b]{0.49\linewidth}
        \centering
        \includegraphics[width=\linewidth]{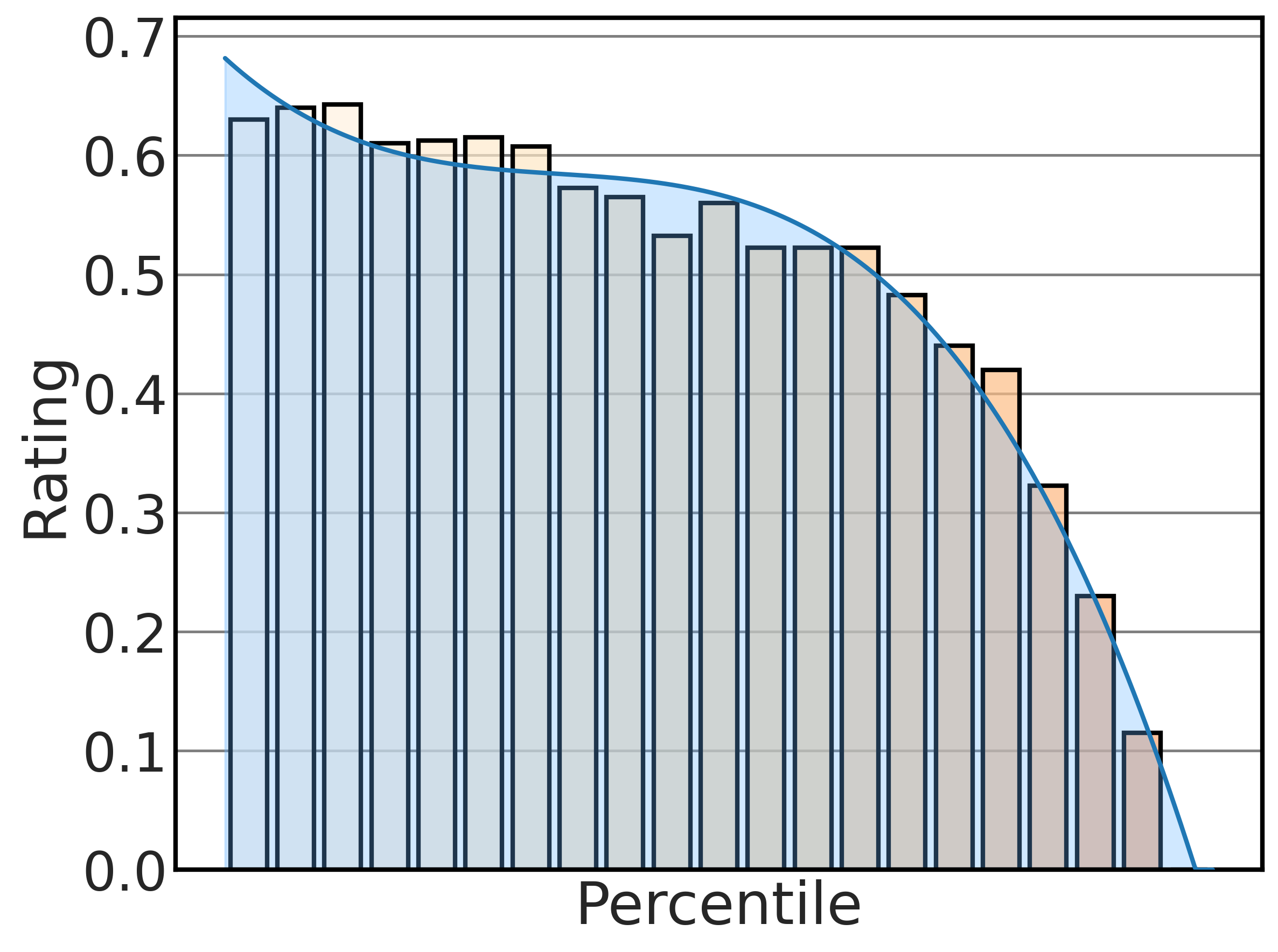}
        \caption{Writing Style}
        \label{fitting:WS}
    \end{subfigure}
    \hfill
    \begin{subfigure}[b]{0.49\linewidth}
        \centering
        \includegraphics[width=\linewidth]{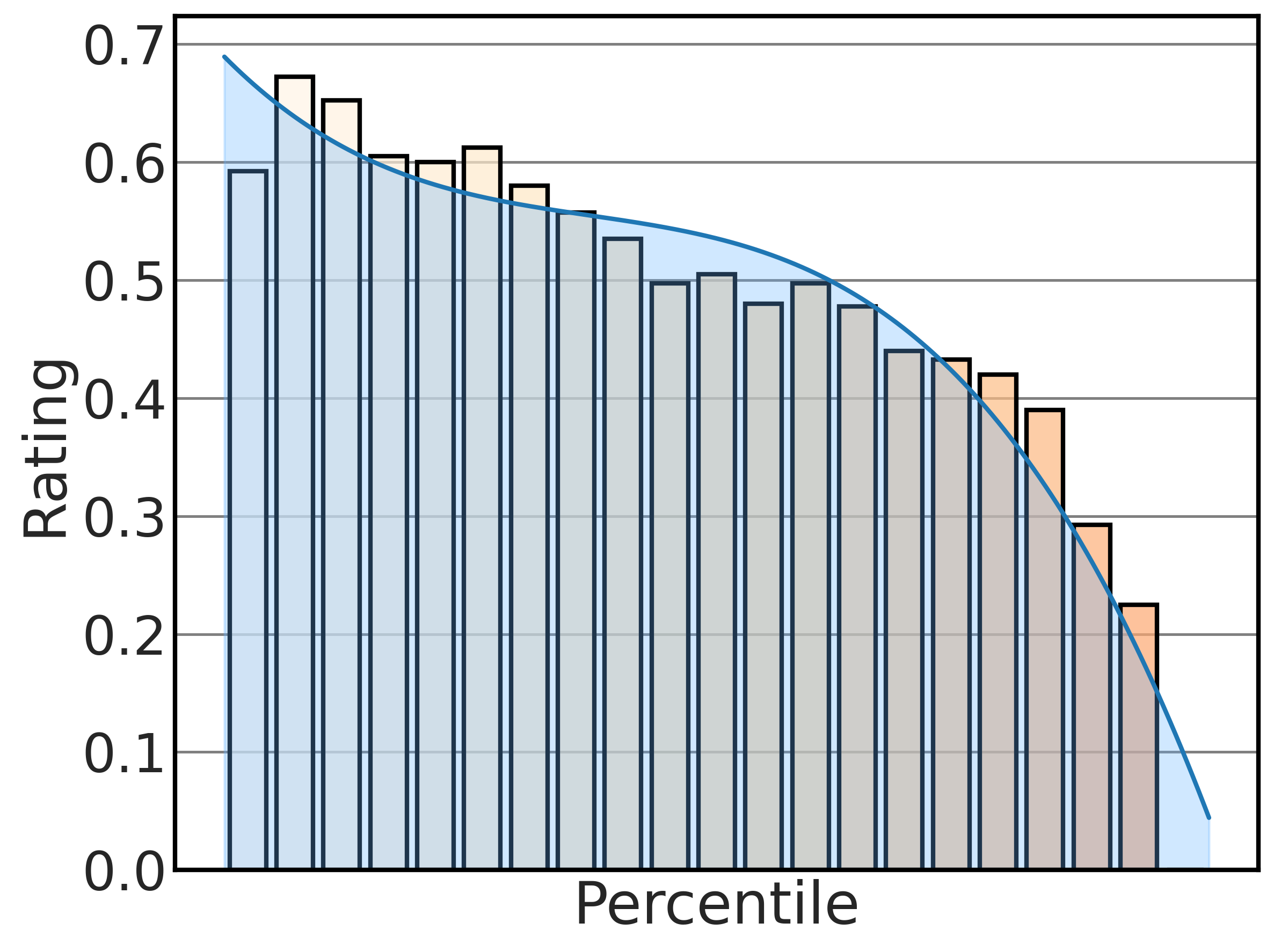}
        \caption{Required Expertise}
        \label{fitting:RE}
    \end{subfigure}
    \vfill
    \begin{subfigure}[b]{0.49\linewidth}
        \centering
        \includegraphics[width=\linewidth]{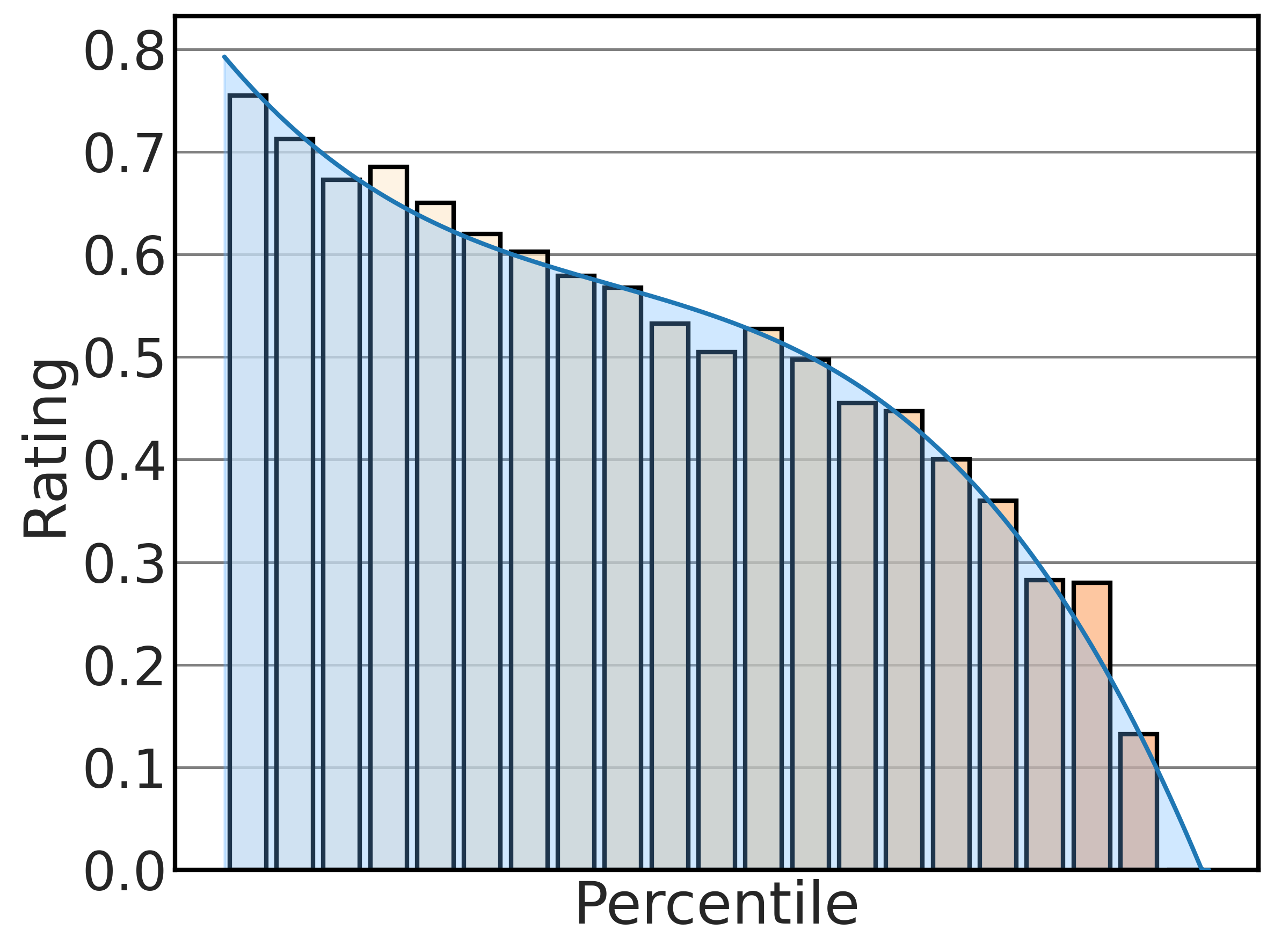}
        \caption{Facts and Trivia}
        \label{fitting:FT}
    \end{subfigure}
    \hfill
    \begin{subfigure}[b]{0.49\linewidth}
        \centering
        \includegraphics[width=\linewidth]{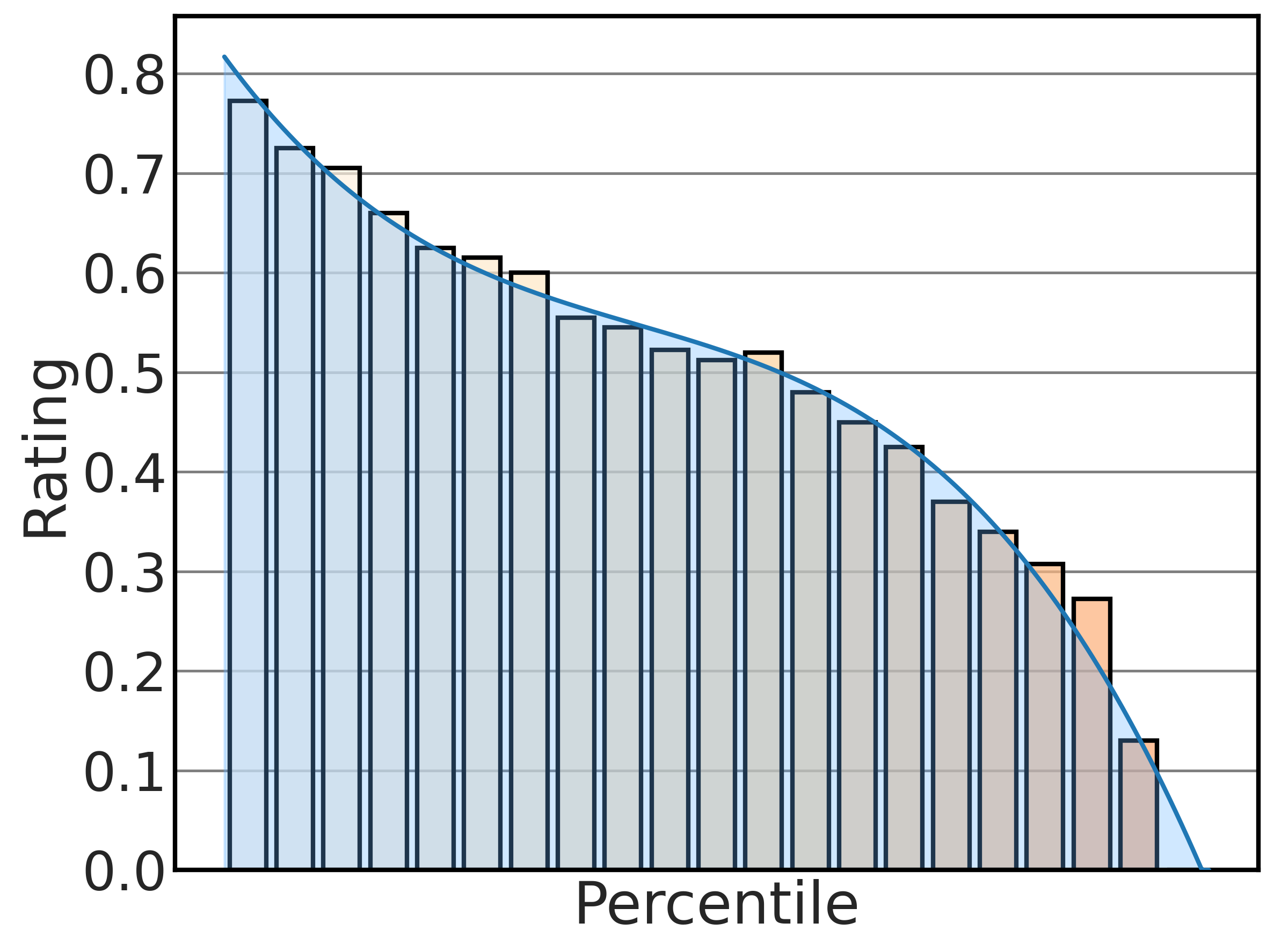}
        \caption{Educational Value}
        \label{fitting:EV}
    \end{subfigure}
    \caption{The win rates of samples in different percentile intervals and the fitted Rating-Percentile curves for 4 raters. }
    \label{fitting}
\end{figure}


\subsection{Formalization of Orthogonality.}
\label{subsec:Formalization of Orthogonality}
\paragraph{Empirical determination of the boundary conditions.}
For convenience and without losing rationality, we use the initial version of the integrated rating calculation formula (without PageRank optimization) to determine the boundary values of orthogonality. The integrated rating of data point \(i\) can be expressed as
\[
I(i) = \sum_{j=1}^{n} \gamma_j o_j A_j(i) 
\]
where \( o_j = \sum_{\substack{k=1 \\ k \neq j}}^{n} O(j, k) \) is a quantification of the overall orthogonality \( R_j \) with other raters. We ignore the impact of raters' reliability and set \(\gamma_j = 1\). By substituting the definition of the overall orthogonality, we obtain
\begin{equation}
\begin{aligned}
    I(i) &= \sum_{j=1}^{n} \left(\sum_{\substack{k=1 \\ k \neq j}}^{n} O(j, k) \right) A_j(i) \\
         &= \frac{1}{2} \sum_{j=1}^{n} \sum_{\substack{k=1 \\ k \neq j}}^{n} O(j, k) \left(A_j\left(i\right) + A_k\left(i\right)\right)
\end{aligned}
\label{eq:6} 
\end{equation}
The final expression, after rearrangement, can be seen as pairwise addition of raters, with their integrated ratings weighted by the orthogonality of the two raters. When two raters are perfectly correlated, it means their ratings are identical across all data points. In this scenario, the information provided by one rater is entirely redundant with respect to the other. Therefore, the orthogonality between these raters should be set to 0, indicating no additional information is gained by considering both ratings. Conversely, when two raters are perfectly orthogonal, it signifies that their ratings are completely independent of each other. Each rater offers unique and complementary perspectives, leading to a comprehensive evaluation when combined. In such cases, the orthogonality should be set to 0.5, reflecting that each rater contributes equally distinct information to the overall rating.

\paragraph{Orthogonality function.}
The most intuitive approach to determine the orthogonality between two raters is to use the correlation coefficient between their rating distributions on a pretraining dataset. The orthogonality-correlation function needs to satisfy two key conditions:

\begin{itemize}
    \item \textbf{Monotonicity:} The stronger the correlation, the lower the orthogonality.
    \item \textbf{Boundary Conditions:} Empirically, when two raters are completely uncorrelated, the orthogonality between them is 0.5; when they are fully correlated, the orthogonality is 0. And the empirical value has been described above.
\end{itemize}


In fact, we choose the orthogonality between two raters to be 0 when they are completely correlated, as we want to avoid duplication issues during rater integration. For instance, consider three raters, where Rater 1 and Rater 2 are completely correlated (correlation coefficient of 1) and are completely uncorrelated with Rater 3 (correlation coefficient of 0). Considering only one iteration, the integration result of the three raters is \( R = O_{12}(R_1 + R_2) + O_{13}(R_1 + R_3) + O_{23}(R_2 + R_3) = (O_l + O_h)R_1 + (O_l + O_h)R_2 + 2O_hR_3 \), where \( O_l = O_{12} \), \( O_h = O_{13} = O_{23} \). Given that Rater 1 and Rater 2 are completely correlated, which means \( R_1 = R_2 \), we obtain \( R = 2(O_l + O_h)R_1 + 2O_hR_3 \). If \( O_l \neq 0 \), \( R_1 \) will be assigned an additional weight. However, one would naturally assume that $R_1$ and $R_3$ should carry the same weight when they are completely uncorrelated.

Considering the scenario where the correlation coefficients among all raters are zero, the calculation of orthogonality would completely degenerate to zero. Therefore, in such extreme cases, we treat all perfectly correlated raters as a single rater and do not proceed with orthogonality-related integration.

Based on the conditions, we explored three functional forms that satisfy the criteria: a linear function, a Gaussian function, and a symmetrically processed Gaussian function. The three functional forms explored for orthogonality are as follows:

(1) \textbf{Linear Function:}
\begin{equation}
O_L(i, j) = \frac{1}{2} \cdot (1 - |r(i,j)|)
\end{equation}

(2) \textbf{Gaussian Function:}
\begin{equation}
O_G(i, j) = \exp\left(-\frac{r(i, j)^2}{2c^2}\right) - \frac{1}{2}
\end{equation}

(3) \textbf{Symmetrically Processed Gaussian Function:}
\begin{equation}
O_{S.G}(i, j) = \left(\frac{3}{2} - |r(i, j)|\right) - \exp\left(-\frac{r(i, j)^2}{2c^2}\right)
\label{linear}
\end{equation}

where \( r(i, j) \) represents the correlation measure between raters \( i \) and \( j \). Specifically, it is quantified using the Pearson correlation coefficient between the score distributions of the two raters over a common dataset. In both Gaussian-based functions, the constant \( c = \sqrt{\frac{1}{2 \ln 2}} \) is determined by the boundary conditions. Figure \ref{orth} illustrates the 3 forms of orthogonality functions with respect to correlation coefficient. 

It's worth noting that the experiments in Appendix \ref{appendix:Orthogonality Function} demonstrate that the symmetrically processed gaussian function exhibits the best performance. Therefore, in the main experiments of this paper, we use the symmetrically processed gaussian function, i.e., Equation (\ref{linear}). Additionally, when integrating the ratings from two raters, their weights based on orthogonality are the same. Therefore, we only use the intrinsic reliability of the raters to weight their ratings.
\begin{figure}[t]
\centering
\includegraphics[width=0.85\linewidth]{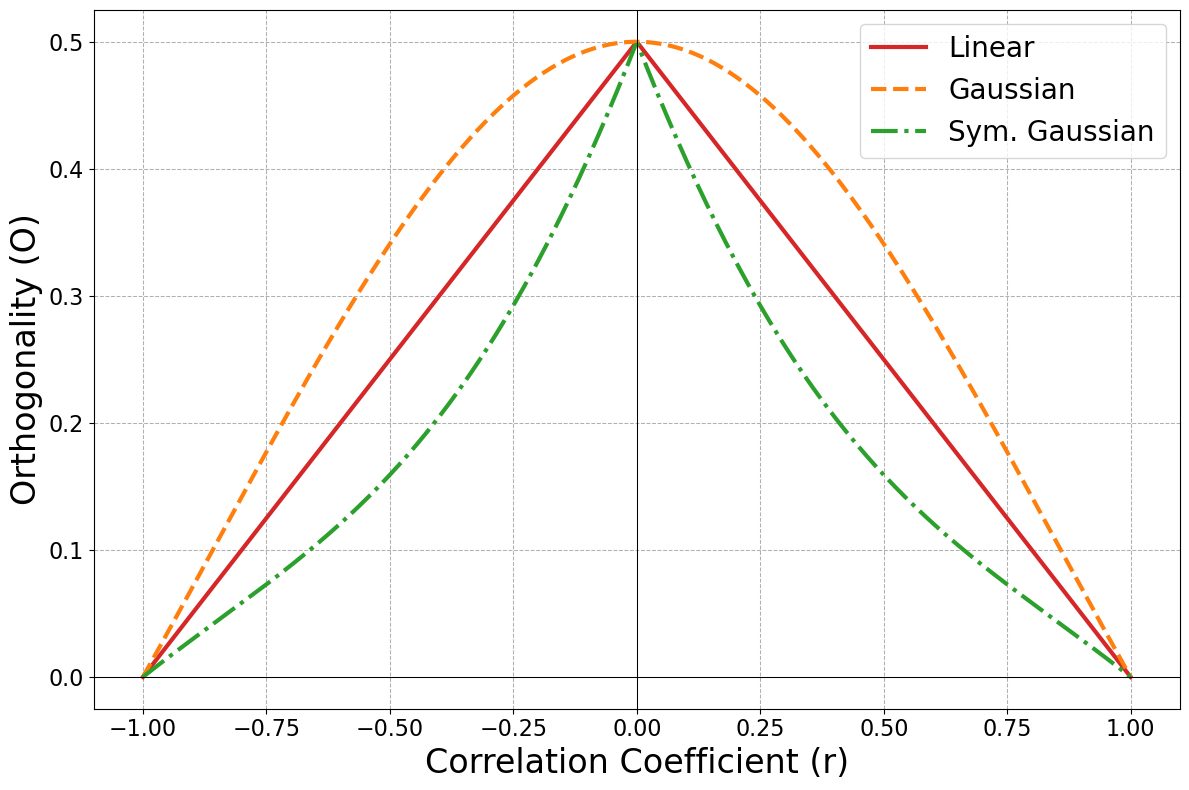} 
\caption{Trends of 3 forms of orthogonality functions with the change in the correlation coefficient.}
\label{orth}
\end{figure}

\begin{table}[t]
    \centering
    \renewcommand\arraystretch{1.5}
    \scalebox{0.8}{
    \begin{tabular}{lr|lr}
        \toprule
            \textbf{Hyperparameter} & \textbf{Value} & \textbf{Hyperparameter} & \textbf{Value} \\
        \midrule
             Attention heads & 16 & Precision & bfloat16 \\
             Layers & 24 & Vocab size & 32,000 \\
             Hidden size & 2048 & Window length & 1024 \\
             Intermediate size & 5504 & Tied embedding & False\\
             Position embedding & ROPE & Activation & SwiGLU \\
        \bottomrule
    \end{tabular}
    }
    \caption{The hyperparameters of model structure.}
    \label{model_structure}
\end{table}

\subsection{Orthogonality-Percentile curve}
\begin{figure}[t]
\centering
\includegraphics[width=0.9\linewidth]{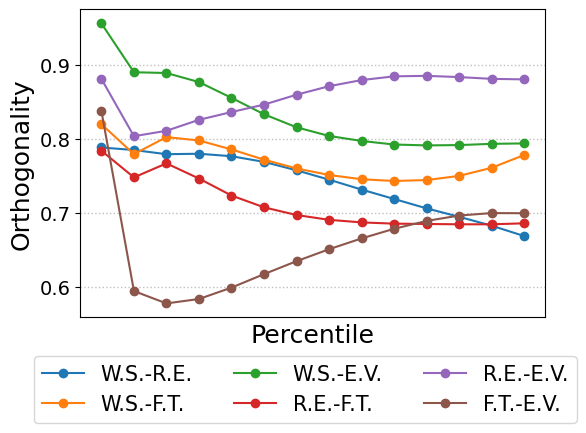} 
\caption{Trends of orthogonality functions with the change in percentile of different pairs of raters.}
\label{ortho}
\end{figure}
Figure \ref{ortho} shows the orthogonality calculated based on data subsets from different quantiles. It is evident that the orthogonality calculated from data in different quantiles varies.

\section{Experimental Details}
\subsection{Model and training} 
\label{subsec:Model and training}
We train a model with 1.3B parameters similar to the Llama architecture, the model structure is detailed in the Table \ref{model_structure}. We train the language model from scratch and randomly initialize the model parameters. We set the batch size to 2048 and the learning rate to 5e-4, using a cosine learning schedule. To accelerate the training and inference processes, we use the bfloat16 format during both training and testing. The training is based on the Megatron framework and utilizes flash attention. The entire model is trained on 16 A100 GPUs for a total of 10,000 steps. For FIRE orthogonality calculation, we use the symmetrically processed gaussian function.

\subsection{Integration Baselines}
\label{subsec:Integration Baselines}
We compare our rating integration method with other baseline methods. Here we give more details about the baselines.

\begin{tcolorbox}[colback=white!95!gray,colframe=gray!50!black,rounded corners,label={prompt-dot}, title={Prompts for Comprehensive Rater}]
\begin{lstlisting}[breaklines=true, xleftmargin=0pt, breakindent=0pt, columns=fullflexible, mathescape, numbers=none]
Compare two text excerpts and choose the text which
1. has a more polished and beautiful writing style.
2. contains more facts and trivia. Prefer specific facts and obscure trivia over more common knowledge.
3. requires greater expertise and prerequisite knowledge to understand it.
4. has more educational value, e.g., it includes clear explanations, step-by-step reasoning, or questions and answers.
Aspects that should NOT influence your judgement: 
1. Which language the text is written in 
2. The length of the text 
3. The order in which the texts are presented 

Note that the texts are cut off, so you have to infer their contexts. The texts might have similar quality, but you should still make a relative judgement and choose the label of the preferred text. 

[Option A] 
{text1} 
[Option B] 
{text2} 

Now you have to choose between either A or B. Respond only with a single word.
\end{lstlisting}
\end{tcolorbox}



\paragraph{Comprehensive Rater} 
Following QuRating \cite{wettig2024qurating}, we collect pairwise comparison data and train a reward model based on Sheared-Llama 1.3B \cite{xiasheared}. We merge all the evaluation criteria into a single prompt to guide GPT-4o comparison. Above is an illustration of the prompt.

\paragraph{Mix Criteria}

We utilize each rater to select the top 20B tokens, which are then merged and duplicates removed. Following this, we randomly pick out another 20B tokens from this merged set. As each sample only requires to excel in one rater's evaluation to be considered for selection, this approach emphasizes the dimension in which the sample performs best amongst all dimensions.

\paragraph{Max Criteria}

Once the scores are aligned, we directly select the dimension with the highest rating to represent the integration result. This method is analogous to performing max-pooling across all dimensions, straightforwardly highlighting the dimension within the sample that has the highest rating.

\paragraph{Average}

For multi-rater integration, the most straightforward approach is to assume that each rater contributes equally to the overall quality of the data. Accordingly, we normalize the ratings provided by each rater for the samples. Following normalization, the average of these ratings is computed to obtain the final integrated rating.

\subsection{Analysis of Computational Cost}
\label{appendix:flops}

{This section analyzes the computational cost (measured in FLOPs) for various data rating and integration methods. The focus is on the cost introduced by applying raters to the pre-training data, excluding the main model’s training cost.  FIRE achieves more effective integration without introducing significant extra cost beyond the base raters. The FLOPs comparison is summarized in Table~\ref{tab:flops_comparison}.}

The analysis includes:
\begin{enumerate}
    \item Rater training (if applicable), 
    \item Rater inference over the entire dataset,
    \item Additional computation for FIRE's win-rate-based integration.
\end{enumerate}

\begin{table}[t]
\centering
\scalebox{0.9}{
\begin{tabular}{l|c}
\toprule
\textbf{Method} & \textbf{FLOPs} \\
\midrule
QuRating (single rater) & $\approx 8.17 \times 10^{20}$ \\
QuRating (mix of criteria, 4 raters) & $\approx 3.26 \times 10^{21}$ \\
FIRE (mix of criteria, 4 raters) & $\approx 3.26 \times 10^{21}$ \\
\bottomrule
\end{tabular}
}
\caption{Comparison of computational cost (FLOPs) across different rating methods.}
\label{tab:flops_comparison}
\end{table}

\paragraph{QuRating (Single Rater).} Each QuRating model is a 1.3B-parameter transformer. It is first fine-tuned on 500K examples (512 tokens each). The total FLOPs for training this rater is approximately:
\[
\text{FLOPs}_{\text{train}} \approx 6 \cdot N \cdot d^2 \cdot L \cdot T  \approx 4.2 \times 10^{19}
\]
Rater inference over the entire 627B-token dataset consumes:
\[
\text{FLOPs}_{\text{infer}} \approx 6 \cdot 627 \times 10^9 \cdot d^2 \cdot L \approx 8.13 \times 10^{20}
\]
\[
\text{FLOPs}_{\text{total}} \approx 8.17 \times 10^{20}
\]
\paragraph{QuRating (Mix of Criteria).} Integrating four QuRating raters (each covering a distinct aspect) requires four forward passes across the full dataset. Since retraining is not needed, total cost is:
\[
\text{FLOPs}_{\text{mix}} \approx 4 \times \text{FLOPs}_{\text{infer}} \approx 3.26 \times 10^{21}
\]

\paragraph{FIRE.} FIRE uses the same four QuRating raters as input, and hence shares the same inference cost. In addition, it estimates the win-rate-percentile mapping using only 20K pairwise comparisons per rater. The additional cost of this step is negligible compared to inference over 627B tokens.

\subsection{Integration of Other Raters}
\label{subsec:ior}
{FIRE has extremely high scalability: our approach is not directly related to the specific attributes of the raters. Therefore, as long as the raters can provide ratings for the samples and are not completely related to each other, they can be integrated using the FIRE method. We also attempt to integrate another raters: QuRating(Required Expertise), QuRating(Facts and Trivia), DSIR-Book, and DSIR-Wiki. Tabel \ref{tab:fire_rater_integration} shows that the average performance of FIRE, after integrating these four raters, is better than each individual rater and the Random baseline, indicating that our integration method is still effective on other raters.}

\begin{table*}[htbp]
\centering
\begin{tabular}{lccccccccc}
\toprule
Method & ARC-E & ARC-C & SciQ & LogiQA & BoolQ & HellaSw. & PIQA & W.G. & AVG. \\
\midrule
Random & 48.2 & 22.3 & 84.5 & 19.7 & 60.8 & 32.1 & 63.5 & 49.2 & 47.5 \\
DSIR with Book & 36.2 & 19.5 & 73.4 & 21.4 & 61.8 & 29.5 & 62.5 & 53.6 & 44.7 \\
DSIR with Wiki & 37.2 & 18.0 & 76.4 & 23.0 & 58.0 & 27.9 & 57.3 & 51.1 & 43.6 \\
QuRating (R.E.) & 50.6 & 23.2 & 83.9 & 22.6 & 61.4 & 30.2 & 59.8 & 49.8 & 47.7 \\
QuRating (F.T.) & 54.1 & 23.0 & 83.5 & 22.0 & 60.9 & 30.4 & 59.5 & 51.7 & 48.1 \\
\textbf{FIRE} & \textbf{57.1} & \textbf{25.9} & \textbf{84.9} & \textbf{20.8} & \textbf{61.5} & \textbf{31.5} & \textbf{60.1} & \textbf{50.3} & \textbf{49.0} \\
\bottomrule
\end{tabular}
\caption{Performance comparison by applying FIRE to integrate four raters: QuRating (Required Expertise), QuRating (Facts and Trivia), DSIR-Book, and DSIR-Wiki.}
\label{tab:fire_rater_integration}
\end{table*}

\subsection{Prompt for Multi-dimension Analysis}
\label{subsec:Multi-dimension analysis prompt}

\label{prompt_for_case_study}
\begin{tcolorbox}[colback=white!95!gray,colframe=gray!50!black,rounded corners,label={prompt-dot1}, title={Prompts for GPT-4o evaluation}]
\begin{lstlisting}[breaklines=true, xleftmargin=0pt, breakindent=0pt, columns=fullflexible, mathescape, numbers=none]
You are a data annotation expert. You should judge that {condition}

Aspects that should NOT influence your judgement: 
1. Which language the text is written in 
2. The length of the text 
3. The order in which the texts are presented 

Note that the texts are cut off, so you have to infer their contexts.
Here is the text:
[TEXT BEGIN]
{text}
[TEXT END]

Please follow the question order to respond. For answer, only respond yes or no.
Return the results for each question in the following json format:
[{
"quesion": "Is this text has a polished and beautiful writing style ?",
"reason": "Fill in the reason for the judgment here",
"answer": "yes/no"
},
...]
\end{lstlisting}
\end{tcolorbox}

We instruct the GPT-4o to evaluate the data selected through the raters, and here we show the prompts. For dimension, we consider four individual dimensions, the same as QuRating \citeauthor{wettig2024qurating} (\citeyear{wettig2024qurating}); and a comprehensive dimension:

\begin{itemize}
    \item Does this text have a polished and beautiful writing style?
    \item Does this text contain many facts and trivia? Prefer specific facts and obscure trivia over more common knowledge.
    \item Does this text have much educational value? E.g., it includes clear explanations, step-by-step reasoning, or questions and answers.
    \item Does this text require a lot of expertise and prerequisite knowledge to understand it?
    \item Does this text contain an informative signal for pretraining a large-language model? An informative data point should be well-formatted, contain some usable knowledge of the world, and strictly NOT have any harmful, racist, sexist, etc. content.
\end{itemize}





\section{Further Analysis of Experiments}

\subsection{Results of different combinations}
\label{subsec:Results of different combinations}
We present the results of different rater combinations in Table \ref{combination results}. The average scores of FIRE (W.S.+R.E.+F.T.) and FIRE (W.S.+R.E.+E.V.) both surpass the results obtained from combining any two of their individual raters. Furthermore, FIRE (W.S.+R.E.+F.T.+E.V.) achieves even better performance.

\begin{table*}[t]
    \centering
    \scalebox{0.85}{
    \begin{tabular}{l|
    >{\centering\arraybackslash}m{1.3cm} 
    >{\centering\arraybackslash}m{1.3cm}
    >{\centering\arraybackslash}m{1cm} 
    >{\centering\arraybackslash}m{1cm}
    >{\centering\arraybackslash}m{1cm}
    >{\centering\arraybackslash}m{1.2cm}
    >{\centering\arraybackslash}m{1cm}
    >{\centering\arraybackslash}m{1cm}
    >{\centering\arraybackslash}m{1cm}
    }
    \toprule
        \textbf{Method} & \textbf{ARC-E} & \textbf{ARC-C} & \textbf{SciQ} & \textbf{LogiQA} & \textbf{BoolQ} & \textbf{HellaSw.} & \textbf{PIQA} & \textbf{W.G.} & \textbf{AVG.} \\ 
    \midrule
        \multicolumn{10}{c}{Two Raters Integration} \\ 
    \midrule
        FIRE (W.S.+R.E.) & 54.5 & 24.6 & 85.6 & 19.7 & 61.3 & 31.3 & 60.4 & 51.8 & 48.7 \\ 
        FIRE (W.S.+F.T.) & 56.1 & 25.0 & 81.5 & 22.7 & 55.1 & 31.7 & 61.2 & 49.7 & 47.9 \\ 
        FIRE (R.E.+F.T.) & 56.9 & 26.1 & 84.7 & 20.7 & 62.0 & 30.5 & 59.4 & 50.0 & 48.8 \\ 
        FIRE (W.S.+E.V.) & 56.7 & 24.2 & 82.2 & 21.0 & 60.3 & \textbf{33.0} & 61.0 & 51.0 & 48.7 \\ 
        FIRE (R.E.+E.V.) & 53.0 & 24.7 & 84.1 & 22.4 & 61.2 & 31.3 & 58.7 & 49.7 & 48.1 \\
    \midrule
        \multicolumn{10}{c}{Three Raters Integration} \\
    \midrule
        FIRE (W.S.+R.E.+F.T.) & 59.0 & 25.4 & 85.6 & \textbf{25.5} & 57.5 & 31.9 & 61.8 & 51.1 & 49.7 \\ 
        FIRE (W.S.+R.E.+E.V.) & 57.8 & 25.9 & 84.5 & 20.6 & 62.0 & 32.7 & 60.5 & 51.3 & 49.4 \\ 
    \midrule
        \multicolumn{10}{c}{Four Raters Integration} \\ 
    \midrule
        FIRE (W.S.+R.E.+F.T.+E.V.) & 59.1 & 26.4 & 86.0 & 21.0 & 61.8 & 32.9 & 59.7 & 52.8 & 50.0 \\
    \bottomrule
    \end{tabular}
    }
    \caption{Downstream tasks results for different rater combinations. We report accuracy for each task, and the best performances are marked in bold. Abbreviations: HellaSw. = HellaSwag, W.G. = WinoGrande, AVG. = Average, W.S. = Writing Style, R.E. = Required Expertise, F.T. = Facts and Trivia, E.V. = Educational Value}
    \label{combination results}
\end{table*}

\subsection{Prompt merge effect} 
\label{subsec:Prompt merge effect}

\begin{table}[t]
    \centering
    \scalebox{0.9}{
    \begin{tabular}{l|cccc}
    \toprule
        ~ & W.S. & F.T. & E.V. & R.E. \\ 
    \midrule
        Sing.(Human) & 53.3 & 69.7 & 85.1 & 92.2 \\
        Sing.(GPT4) & 52.1 & 69.7 & 85.5 & 91.8 \\ 
        Comp.(GPT4) & 57.8 & 80.5 & 86.9 & 50.3 \\
    \midrule
        $\boldsymbol{\rho_{Human-Sin.(GPT4)}}$ & 0.81 & 0.85 & 0.86 & 0.77 \\ 
        $\boldsymbol{\rho_{Human-Com.(GPT4)}}$ & 0.72 & 0.64 & 0.72 & 0.32 \\ 
    \bottomrule
    \end{tabular}
    }
    \caption{Results of GPT-4 and human annotation for 3,000 samples. Sing. = Single, Comp. = Comprehensive.}
    \label{comprehensive_single}
\end{table}

We investigate whether combining multiple rules within a single prompt can effectively meet the evaluation standards of each rule. We randomly selected 3,000 data points and guided GPT-4's evaluation using a prompt that integrates multiple rules, alongside conducting individual rule-guided evaluations and separate human annotator evaluations for each criterion. To understand how GPT-4's holistic evaluation aligns with each individual dimension, we employed CoT \cite{wei2022chain} approach: the model first evaluates each rule separately, and then provides an overall evaluation. Each rule is assessed with a binary yes/no question, and after six evaluations, we average the results to obtain the final score (for human evaluations, this entails averaging the scores given by six annotators).

Table \ref{comprehensive_single} displays the percentage of data with a relatively high degree in each dimension, along with the correlation coefficients between GPT-4's and human evaluations. From the proportions of high-quality data in each dimension, it is evident that GPT-4's scores approximate human scores more closely when evaluated individually, with differences within a margin of 0.2 at most. However, GPT-4's scores exhibit greater variability when multiple rules are integrated. In terms of correlation coefficients, there is a strong correlation between GPT-4's individual scoring and human scoring, but this correlation significantly diminishes when it comes to comprehensive scoring. Particularly, in the 'Require Expertise' aspect, the correlation is only 0.32. This suggests that GPT-4's current capability to adhere to all rules in a single prompt is still inadequate.

\subsection{Orthogonality Function}
\label{appendix:Orthogonality Function}


\begin{figure}[t]
    \centering
    \includegraphics[width=\linewidth]{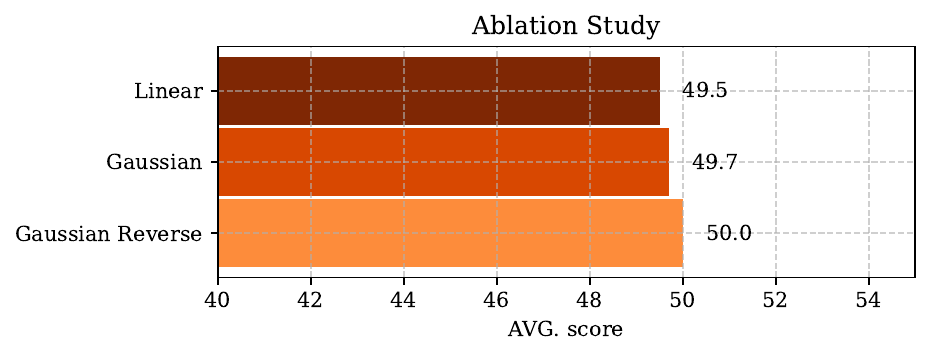}
    \caption{Ablation experiments evaluating the impact of different orthogonality functions on model performance.}
    \label{func_ablation study}
\end{figure}

There are various methods to calculate orthogonality, and we aim to identify the most effective one. We compare three functions for orthogonality calculation: \textit{linear}, \textit{gaussian}, and \textit{sym. gaussian }(the symmetrically processed gaussian function). As illustrated in Figure \ref{func_ablation study}, the three functions achieve comparable results, all surpassing the configuration without orthogonality, which demonstrates the effectiveness of incorporating orthogonality. Besides, the sym. gaussian function outperforms the other two. The gaussian function smooths the correlation coefficient around zero, while the sym. gaussian function amplifies changes in orthogonality near zero. The linear function, however, strikes a balance between these two. Our experimental results confirm that enhancing the rate of change around zero is more efficient, emphasizing the role of highly orthogonal raters, and intensifying the penalties for raters with high correlation.

\subsection{Effect of Sample/top-K}


\begin{figure}[htbp]
    \centering
    \includegraphics[width=1.0\linewidth]{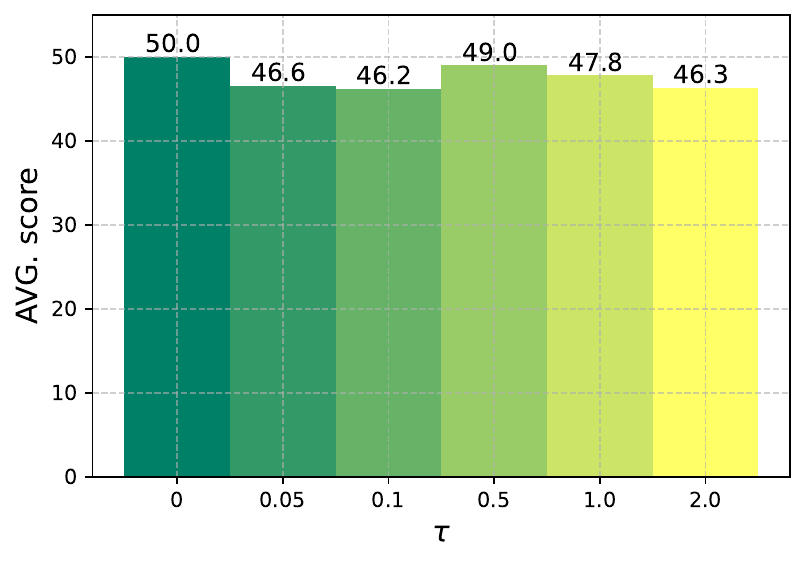}
    \caption{The impact of the sample temperature coefficient on model performance. $\tau$ is the temperature coefficient, and when $\tau=0$, it refers to top-K selection.}
    \label{fig:sample}
\end{figure}

Qurating \cite{wettig2024qurating} suggest that sampling is more effective than directly selecting the top-K data. In this experiment, we integrate four different raters to rigorously investigate the impact of sampling on the model's final performance. Specifically, we calculate the sampling probability for all rated samples using the following softmax formula:

\begin{equation}
    P(x) = \frac{exp(I(x)/\tau)}{\sum{exp(I(x)/\tau)}}
\end{equation}

where $\tau$ is the temperature parameter. We train a model of the same size as in the previous experiments. From the results presented in Figure \ref{fig:sample}, several key observations can be made:

(1) Our findings reveal that direct top-K selection outperforms sampling, further affirming the efficacy of our integration method. Additionally, \citeauthor{wettig2024qurating} (\citeyear{wettig2024qurating}) posits that sampling enhances data diversity, which is beneficial for model learning. Our results indicate that the top-K scores are higher than the sampling scores, demonstrating that the top-ranked data according to our integration rating exhibit a broader distribution rather than being concentrated in a single domain.

(2) Contrary to the optimal value of 2 suggested by \citeauthor{wettig2024qurating} (\citeyear{wettig2024qurating}), our analysis indicates that a relatively smaller $\tau$ value yields optimal results. A smaller $\tau$ accentuates the impact of ratings, suggesting that our method effectively selects higher-quality data. Furthermore, our multi-dimensional approach also accounts for the diversity of data types, ensuring a more comprehensive evaluation.

\section{Selected Data Cases}

We show the document cases rated by the single raters and FIRE in the Table \ref{data_cases}. For each method, we show the best/middle/worst sample.

\begin{table*}[t]
    \centering
    \scalebox{0.85}{
    \begin{tabular}{l|
    >{\arraybackslash}m{4.1cm}|
    >{\arraybackslash}m{4.1cm}|
    >{\arraybackslash}m{4.1cm}
    }
    \toprule
        Method & Best & Middle & Worst \\ \midrule
        Writing Style & {\tiny ... is the very thing that makes each person inimitable, the thing that allows us finally to see and celebrate one another's distinct natures. Once it is understood that this expanse must always exist, each person is free to become whatever it is they will become, unburdened by the need to shape themselves to fit their partner. And this individuation need not be a growing apart. For if each partner can remember the beauty and necessity of the expanse, then they can come to appreciate fully the ...} & {\tiny ... him some of that history," she said. Near the beginning of the session panelist Don Lemon of CNN played video of his story on what he called "a picture of racial unity," Sherrod's reunion with the elderly white farmers whose farm she helped save, after first not making a full effort at a nonprofit where she worked 24 years ago. The author of legislation that would require natural-gas companies to disclose hydraulic-fracturing fluids says she feels betrayed by industry groups that have spoken ...} & {\tiny ... - 2006 Y-T-D Stat Scoring Average (Actual) - 2006 Stat Scoring Average (Actual) - 2006 Y-T-D Stat Scoring Average (Actual) - 2006 Stat Scoring Average (Actual) - 2006 Y-T-D Stat Scoring Average (Actual) - 2006 Stat Scoring Average (Actual) - 2006 Y-T-D Stat Scoring Average (Actual) - 2006 Stat Scoring Average (Actual) - 2006 Y-T-D Stat Scoring Average (Actual) - 2006 Stat Scoring Average (Actual) - 2006 Y-T-D Stat Scoring Average (Actual) - 2006 Stat Scoring Average (Actual) - 2006 Y-T-D St ...} \\ \midrule
        Required Expertise & {\tiny ... induced climate change have used instrumental records to study how quickly climate is warming across different parts of the world. Our study uses a collection of natural archives that preserve information about past temperatures over a much longer period, spanning the last 500 years, to ask the question: "When did the sustained warming trends that we've seen in the 20th and 21st Centuries first begin?" ...} & {\tiny ... (by R.M. Butler?) in 'Folder 20' seen by Nick Sheaff, 1970s; Liam Swords, Achonry and its churches (Strasbourg: Éditions du Signe, 2007), 78(illus.). Nature: Additions, for Lady Fitzgerald Arnott. contractor: Michael O'Brien, Dun Laoghaire. Refs: IB 59, 27 Oct 1917, 551; MS letter in IAA (Acc. 88/118) from Rev.P. Kilkenny to Butler, 6 Jun 1919, states that he cannot resist 'revolutionary, ...} & {\tiny ... more. Awesome. : 1138 This is one awesome article.Thanks Again. Really Cool. : 1136 Appreciate you sharing, great article. : 1135 I cannot thank you enough for the blog post.Much thanks again. Great. : 1134 I think this is a real great post.Much thanks again. Cool.: 1133 A big thank you for your blog article.Really looking forward to read more. Great. : 1132 Great, thanks for sharing this blog post.Much thanks again. ...} \\ \midrule
        Facts and Trivia & {\tiny ... Native American to carry the United States flag at the opening ceremony of the Olympic Games: Clarence "Taffy" Abel (Chippewa).  1926  First Native American in the NHL New York Rangers November 16, 1926: Clarence "Taffy" Abel (Chippewa).  First Native American woman to hold state office in Oklahoma: Jessie Elizabeth Randolph Moore (Chickasaw). First national reform group with only Native American membership: National Congress of American Indians (NCAI) by Zitkala-Sa ...} & {\tiny ... I probably only understand two-thirds of what's going on, honestly, but it's still damn good stuff. (Audiobook) Being Mortal by Atul Gawande: Heard raves about this non-fiction about elder and end-of-life care from many a Rioter, and it's ringing all my Books That Make Me Want To Change the World buttons. (audiobook) Half-Resurrection Blues by Daniel José Older: Loved Older's Salsa Nocturna, so I speedily picked up this novel about a half-alive, half-dead sort-of-secret-agent who works for ...} & {\tiny ... T'AIME JE T'AIME JE T'AIME JE T'AIME JE T'AIME JE T'AIME JE T'AIME JE T'AIME JE T'AIME JE T'AIME JE T'AIME JE T'AIME JE T'AIME JE T'AIME JE T'AIME JE T'AIME JE T'AIME JE T'AIME JE T'AIME JE T'AIME JE T'AIME JE T'AIME JE T'AIME JE T'AIME JE T'AIME JE T'AIME JE T'AIME JE T'AIME JE T'AIME JE T'AIME JE T'AIME JE T'AIME JE T'AIME JE T'AIME JE T'AIME JE T'AIME JE T'AIME JE T'AIME JE T'AIME JE T'AIME JE T'AIME JE T'AIME JE T'AIME ...} \\ \midrule
        Educational Value & {\tiny ... Learn what "big history" is and how scholars apply this approach to the story of humanity. Gain new understanding of the complete sweep of human history, across all civilizations and around the world. Use the lens of history to find out what makes us human, why the world exists as it does today, and where we might be going in the future. See how the environment, population growth, social complexity, and more have driven the rise and fall of civilizations over ...} & {\tiny ... a similar set of models as Fig. 2, this time displaying the C-star fractions for models with varying $f_{\rm CE}$  between 0.008 and 0.4 at the base of the convective envelope. In this instance, there is a far more straightforward interpretation, with an increase in $f_{\rm CE}$  producing an increased C-star fraction, in almost all cases. Furthermore, there is more readily acceptable agreement with the observed C-star fractions than was the case for the ...} & {\tiny ... 05:17:33 https://cse.google.com.bo/url?sa=t url=https://toppornsites.mobi 2023-01-27 05:17:33 https://cse.google.com.br/url?sa=t url=https://toppornsites.mobi 2023-01-27 05:17:33 https://cse.google.com.by/url?sa=t url=https://toppornsites.mobi 2023-01-27 05:17:33 https://cse.google.com.bz/url?sa=t url=https://toppornsites.mobi ...} \\ \midrule
        FIRE & {\tiny In mathematics, the differential geometry of surfaces deals with the differential geometry of smooth surfaces with various additional structures, most often, a Riemannian metric. Surfaces have been extensively studied from various perspectives: extrinsically, relating to their embedding in Euclidean space and intrinsically, reflecting their properties determined solely by the distance within the surface as measured along curves on the surface. One of the fundamental concepts investigated is ...} & {\tiny ... Everett, Washington: Charlotte Murray, 2010. Second Edition of 10. 7.5 x 5.25"; 56 pages. Images captured using three digital cameras, a Nikon Coolpix 5700, a Nikon D70, and a Nikon D80. Printed with Epson Photo Stylus R2880 printer with UltraChrome K3 pigment inks on Epson Premium Presentation Paper Matte. Papyrus font. Coil binding with green see through cover and lightweight cardboard back. Colophon: "The Dead Tree Scrolls first edition was created in 2005. This second edition was ...} & {\tiny ... HBP - MILLER; MORALES. SF - WILKERSON(2). SB - EPPS(9); HORAN(3); PINDER(4). CS - EPPS(3). Clemson                IP  H  R ER BB SO AB BF Justin Sarratt......  6.1  6  3  3  1  8 24 26 Alex Frederick......  0.1  0  0  0  0  0  1  1 Joseph Moorefield...  0.1  0  0  0  0  1  1  1 Matt Campbell.......  2.0  0  0  0  3  3  6  9 Virginia Tech          IP  H  R ER BB SO AB BF Marc Zecchino.......  7.1  7  5  4  1  7 27 29 Jake Atwell.........  0.2  3  3  3  0  0  5  5 Sean McDermott ...} \\ 
    \bottomrule
    \end{tabular}
    }
    \caption{Data cases are randomly selected from the top, middle, and bottom 0.01\% of training samples, representing the best, middle, and worst cases for each method.}
    \label{data_cases}
\end{table*}

\end{document}